\documentclass[11pt]{article}
\pdfoutput=1
\usepackage{fullpage}
\usepackage[margin=1in]{geometry}
\usepackage[utf8]{inputenc}

\usepackage{graphicx}

\usepackage{amsmath}
\usepackage{amssymb}
\usepackage{amsthm}

\usepackage{thmtools}
\usepackage{thm-restate}
\usepackage[vlined, ruled]{algorithm2e}
\usepackage{bbm}
\usepackage{bm}
\usepackage{verbatim}
\usepackage[pagebackref]{hyperref}
\hypersetup{colorlinks=true,citecolor=blue,linkcolor=blue}
\usepackage{array}
\usepackage{caption}
\usepackage{subcaption}
\usepackage{float}
\usepackage{changepage}
\usepackage{enumitem}
\usepackage{cases}
\usepackage[dvipsnames]{xcolor}
\usepackage{cleveref}

\usepackage{xfrac}
\usepackage{tikz, pgfplots}

\usetikzlibrary{decorations.pathreplacing}

\usepackage{natbib}

\renewcommand{\hat}{\widehat}
\renewcommand{\tilde}{\widetilde}

\DeclareMathOperator*{\argmax}{arg\,max}
\DeclareMathOperator*{\argmin}{arg\,min}

\newcommand{\diff}{\text{d}}

\newcommand{\state}{\theta}
\newcommand{\statesp}{\Theta}

\newcommand{\prob}{p}

\newcommand{\expect}[2]{{\mathbf{E}}_{#1}\left[#2\right]}

\newcommand{\variance}[2]{{\mathbf{Var}}_{#1}\left[#2\right]}

\newcommand{\score}{S}

\newcommand{\cdl}{\textsc{CDL}}

\newcommand{\ECE}{\textsc{ECE}}
\newcommand{\CDL}{\textsc{CDL}}

\newcommand{\ind}[1]{\mathbb{I}\left[#1\right]}

\newcommand{\pr}[1]{Pr\left[#1\right]}

\newcommand{\swapS}{\textsc{Swap}}

\newcommand{\pred}{\prob}

\newcommand{\vpred}{\bm{\pred}}
\newcommand{\vstate}{\bm{\state}}

\newcommand{\reals}{\mathbb{R}}

\newcommand{\distcal}{\textsc{DTC}}

\newcommand{\smooth}{\textsc{smCal}}

\newcommand{\empq}[1]{\hat{q_{#1}}}

\newcommand{\btlp}{b}
\newcommand{\ftlp}{q}
\newcommand{\ftplp}{p}

\newcommand{\btplp}{r}

\newcommand{\ftplvec}{\bm{p}}
\newcommand{\ftlvec}{\bm{q}}

\newcommand{\ftplpred}{P}
\newcommand{\ftlpred}{Q}
\newcommand{\btlpred}{B}
\newcommand{\btplpred}{R}

\newcommand{\ftldm}{\mathcal{Q} }

\newcommand{\rate}{\tau}

\newcommand{\ploss}{\textsc{DL}}

\newcommand{\distance}{\textsc{Dist}}

\newcommand{\alg}{f}

\newcommand{\empiricalq}{\tilde{q}}

\newcommand{\empiricalQ}{\tilde{Q}}

\newcommand{\density}{\rho}

\newcommand{\dtv}{d_{\text{TV}}}

\newcommand{\supp}{\text{supp}}

\newcommand{\dpe}{\gamma}
\newcommand{\dpd}{\delta}

\newcommand{\dpmech}{\mathcal{M}}

\newtheorem{theorem}{Theorem}[section]
\newtheorem*{theorem*}{Theorem}
\newtheorem{lemma}[theorem]{Lemma}
\newtheorem{definition}[theorem]{Definition}

\newtheorem{proposition}[theorem]{Proposition}
\newtheorem{corollary}[theorem]{Corollary}

\newtheorem{claim}[theorem]{Claim}

\newtheorem*{example*}{Example}

\date{}

\title{Smooth Calibration and Decision Making}
\author{
    Jason Hartline\\
    Northwestern University\\
    Computer Science\\
    \texttt{hartline@northwestern.edu}
    \and 
    Yifan Wu\\
    Northwestern University\\
    Computer Science\\
    \texttt{yifan.wu@u.northwestern.edu}
    \and
    Yunran Yang\\
    Shanghai Jiaotong University\\
    Zhiyuan College\\
    \texttt{yyr0816@sjtu.edu.cn}
}

\begin{document}

\maketitle
\begin{abstract}
  Calibration requires predictor outputs to be consistent with their Bayesian posteriors. For machine learning predictors that do not distinguish between small perturbations, calibration errors are continuous in predictions, e.g.\ smooth calibration error \citep{foster2018smooth}, distance to calibration \citep{utc}. On the contrary, decision-makers who use predictions make optimal decisions discontinuously in probabilistic space, experiencing loss from miscalibration discontinuously. Calibration errors for decision-making are thus discontinuous, e.g.,\ Expected Calibration Error \citep{foster1997calibrated}, and Calibration Decision Loss \citep{hu2024predict}. Thus, predictors with a low calibration error for machine learning may suffer a high calibration error for decision-making, i.e.\ they may not be trustworthy for decision-makers optimizing assuming their predictions are correct. It is natural to ask if post-processing a predictor with a low calibration error for machine learning is without loss to achieve a low calibration error for decision-making. In our paper, we show post-processing an online predictor with $\epsilon$ distance to calibration achieves $O(\sqrt{\epsilon})$ $\ECE$ and $\cdl$, which is asymptotically optimal. The post-processing algorithm adds noise to make predictions differentially private. The optimal bound from low distance to calibration predictors from post-processing is non-optimal compared with existing online calibration algorithms that directly optimize for $\ECE$ and $\cdl$. 
\end{abstract}

\section{Introduction}

Calibration requires that predictions are empirically conditionally unbiased. Consider a sequence of predictions for the chance of rain, a predictor is calibrated if, for every $p$, among the days that the prediction is $p$, the fraction of rainy days is also $p$. Calibrated predictions can thus be reliably interpreted as probabilities. 

Calibration errors quantify the error of a predictor from being perfectly calibrated. Machine learning (ML) predictors make predictions continuously in probabilistic space, so 
calibration errors for ML are continuous in prediction values and do not distinguish between small perturbations in predictions. Two canonical examples are the \textit{smooth calibration error} \citep{foster2018smooth} and the \textit{distance to calibration} ($\distcal$) \citep{utc}. As an illustrating example of the calibration errors for ML, consider a predictor in \Cref{tab: intro example predictor}.
\begin{table}[h]
    \centering
    \begin{tabular}{c|c|c}
Prediction value    & $\#$ days & conditional frequency of rain \\
\hline
    $50.01\%$     &  half of the days & $0$\\
    $49.99\%$ & half of the days & $1$
    \end{tabular}
    \caption{A miscalibrated predictor for the chance of rain.}
    \label{tab: intro example predictor}
\end{table}
Although the predictions of $50.01\%$ and $49.99\%$ are biased, the total number of rainy days is $50\%$, indicating the predictor is very close to a calibrated predictor that always outputs $50\%$. Both $\distcal$ and the smooth calibration error are about $0.01\%$, close to $0$. The smooth calibration error combines the bias over all the days by weighing biases continuously, e.g.\ weighing bias $(50.01\% - 0)$ by $-0.01\%$, $(49.99\% - 1)$ by $0.01\%$, and summing together (the weights are Lipschitz continuous in prediction values). The smooth calibration error is linearly related to $\distcal$, which calculates the expected $\ell_1$ distance between the predictor and the nearest calibrated predictor, which in this example predicts $50\%$ every day. 

Decision-makers make decisions discontinuously in probabilistic space, thus, a calibration error for decision-making is discontinuous in the prediction space. 
For example, consider a decision problem with binary action space, bringing an umbrella or not. The decision maker receives a payoff of $1$ when the decision matches the state, i.e.\ bringing an umbrella when rainy, not bringing when not rainy, and a payoff of $0$ in other cases. When assisted by a prediction, the action of a decision-maker changes from not bringing an umbrella to bringing an umbrella at the prediction threshold of $50\%$. Two examples of calibration errors for decision-making are Expected Calibration Error ($\ECE$) \citep{foster1997calibrated} and Calibration Decision Loss ($\cdl$) \citep{hu2024predict}. $\cdl$ quantifies the worst-case decision loss of a decision-maker who trusts the prediction as a probability, where the worst-case is taken over all payoff-bounded decision tasks. By definition, $\cdl$ upperbounds any decision-maker's loss. $\ECE$, the most well-studied calibration error metric, is defined by the averaged absolute bias in predictions. For example, $\ECE$ averages over $|50.01\% - 0|$ and $|49.99\% - 1|$ for the predictor in \Cref{tab: intro example predictor} and has a calibration error of $50.01\%$. \citet{kleinberg2023u} shows that $\ECE$ linearly upperbounds the decision loss of every payoff-bounded decision task, implying an upperbound of $\cdl$. 

From the decision-making perspective, having a low calibration error for ML, however, does not guarantee a low calibration error for decision-making or being trustworthy for decision-making. Consider the same example of a predictor in \Cref{tab: intro example predictor} and the umbrella decision problem above. 
According to a calibration error for ML, e.g.\ distance to calibration, the predictor is $0.01\%$ close to a calibrated predictor that always outputs $50\%$. However, to the decision-maker, the prediction suggests not taking an umbrella when the weather is rainy, and taking an umbrella when not rainy. This non-trustworthiness comes from the discontinuity of decision-making which the decision-maker changes an action at the threshold $50\%$.

Here is the natural question: can we design a post-processing algorithm that, given any predictor with a low calibration error for machine learning, outputs predictions with a low calibration error for decision-making? Ideally, the post-processing algorithm should achieve near-optimal guarantees that asymptotically match the guarantees from directly optimizing for decision-making. 

Our paper designs a post-processing algorithm that, given any predictor with $\distcal = \epsilon$, outputs differentially private predictions with $\ECE$ and $\cdl$ bounded by $O(\sqrt{\epsilon})$, in both the batch setting and the online setting. We give lower bounds, described below, for both that online and batch setting, that show that this post-processing algorithm is asymptotically optimal.  Additionally the online lower bounds shows that the optimal predictors for decision makers cannot be constructed from optimal predictors from machine learning.

We show that the privacy-based post-processing algorithm is asymptotically optimal in the online setting. This optimality implies there does not exist a post-processing algorithm that achieves the same guarantee as known online algorithms that directly optimize predictions for $\ECE$ and $\cdl$. For online calibration, there has been shown an $O\left(T^{-\frac{1}{3}-c}\right)$ ($c>0$) upperbound on optimal algorithm for $\ECE$ \citep{dagan2023external}, a $\Tilde{O}(T^{-\frac{1}{2}})$ optimal bound to $\cdl$ \citep{hu2024predict}, and an $\Omega(T^{-\frac{2}{3}})$ lowerbound to $\distcal$ \citep{qiao-distance}. Thus, applying the lowerbound of $\Omega(\sqrt{\epsilon})$, any post-processing algorithm can only achieve the non-optimal $\Omega(T^{-\frac{1}{3}})$ $\ECE$ and $\cdl$.

We show that the privacy-based post-processing algorithm is asymptotically optimal in the batched setting in two models. The first model considers post-processing algorithms applied individually to each prediction, and the same guarantee and lowerbound to $\ECE$ and $\cdl$ applies as in the online setting. The second model allows algorithms that post-process the entire batch of predictions.  However, doing so just to attain calibration is too easy: simply ignoring the individual information in each prediction and averaging them all will be close to calibrated.  Thus, we impose a stronger benchmark that measures the worst-case decision loss relative to a nearby — in the sense of $\epsilon$ Distance to Calibration — calibrated predictor.  This worst case is taken over all such nearby calibrated predictors and all bounded decision problems. We show that the privacy-based post-processing algorithm achieves $O(\sqrt{\epsilon})$ decision loss and that this result is tight, i.e.\ no other post-processing algorithm achieves asymptotically better decision loss.  

\subsection{Related Work}

\paragraph*{Calibration Error Metrics.}
The most relevant work to ours,  \citet{smoothECE}, introduces the error metric Smooth ECE, which, given a predictor, calculates the $\ECE$ with Gaussian noise added to the predictions. For any predictor with $\distcal = \epsilon$, smooth ECE is shown to be bounded by $\Theta(\sqrt{\epsilon})$. Instead, our paper focuses on the decision-making perspective of calibration. We show that this bound of $\Theta(\sqrt{\epsilon})$ is tight, suggesting that from a decision-making perspective, optimizing for $\distcal$ and post-processing achieves suboptimal guarantees. Our post-processing algorithm also generalizes the result of \citet{smoothECE} by considering noise distributions for differential privacy. 

As introduced previously, existing calibration error metrics mainly focus on two aspects: calibration errors for machine learning, continuous in predictions, e.g.\ smooth calibration error \citep{foster2018smooth}, distance to calibration\footnote{We follow \citet{qiao-distance} and refer to \textit{distance to calibration} as the \textit{lower} distance to calibration in \citet{utc}.} \citep{utc}, smooth ECE \citep{smooth}; and calibration errors for decision-making, e.g. the canonical $\ECE$ \citep{foster1997calibrated} and the Calibration Decision Loss \citep{hu2024predict}. Recently, as an orthogonal property to continuity and decision-making, \citet{haghtalab2024truthfulness} propose an approximately truthful calibration error metric for an expected-error-minimizing sequential predictor. 

\paragraph*{Online Calibration.} In online calibration, the predictor repeatedly interacts with an adversary selecting a binary state. In each round, both the predictor and the adversary know the history of predictions and states, but are not allowed to strategize conditioned on the opponent's action in the current round. \citet{foster1998asymptotic} showed an upperbound of $\ECE = O(T^{-\frac{1}{3}})$, which is recently proven to be polynomial-time achievable by \citet{noarov2023highdimensional}. Recently,  \citet{dagan2024breakingt23barriersequential} improves the upperbound to $O\left(T^{-\frac{1}{3}-c}\right)$ for some constant $c>0$. On the lowerbound side, \citet{sidestep} showed there exists an $O(T^{-0.472})$ lowerbound, strictly above $\Tilde{O}(\frac{1}{\sqrt{T}})$, which is improved to $O(T^{-0.456})$ by \citet{dagan2024improved}. 

For linearly related smooth calibration error and $\distcal$, \citet{qiao-distance} prove an $O(\frac{1}{\sqrt{T}})$ upperbound and an $O(T^{-\frac{2}{3}})$ lowerbound.  \citet{arunachaleswaran2025elementary} design a simple polynomial-time algorithm that achieves $\distcal = O(\frac{1}{\sqrt{T}})$. 

The Calibration Decision Loss ($\cdl$) is introduced in \citet{hu2024predict} with a bound of $\Tilde{O}(\frac{1}{\sqrt{T}})$, tight up to a logarithmic factor. 

\paragraph*{Omniprediction.} 
Our definition of decision loss for the batch setting can be equivalently formulated as achieving omniprediction with regard to reference predictors and a set of loss functions. Calibration guarantees the trustworthiness of predictions by every decision-maker, allowing decision-making to be separated from predictions. Introduced in \citet{omni}, omnipredictor follows the same idea, requiring an omnipredictor to achieve a comparable guarantee with regard to a class of loss functions and a set of competing predictors. Techniques from the algorithmic fairness literature, e.g. \citet{mc, ma}, have been applied to achieve omniprediction in both online and batch settings \citep{omni, omni-regression, characterize-omni, oracle-omni, constrained}. While the classical guarantee usually learns an omnipredictor that competes with the hypothesis space of predictors, our decision loss evaluates a predictor with regard to the set of calibrated predictors close in $\distcal$.

\section{Preliminaries}\label{sec: prelim}

\paragraph*{Mathematical Notations.} We write $D_{X, Y}$ as the joint distribution between random variables $X$ and $Y$, and $X\sim D$ as random variable $X$ drawn from distribution $D$. Where it is obvious from the context, we write $\Pr[X = x]$ for the probability of a discrete random variable as well as the probabilistic density function of a continuous random variable. 

We consider a prediction problem of a binary state $\state\in \statesp = \{0, 1\}$. A predictor is specified by a joint distribution $D_{\ftplpred, \statesp}$ over the prediction $\ftplp$ and the state $\state$. Slightly abusing the notation, we also write a predictor as a random variable $\ftplpred$, omitting the state, where a realized prediction value is $\ftplp$.

Our privacy-based post-processing algorithm adds noise to make predictions differentially private. \Cref{def: differential privacy} defines a differentially private mechanism for predictions. 
\begin{definition}[Differential Privacy]\label{def: differential privacy}
    A mechanism $\dpmech$ is $(\dpe, \dpd)$-differentially private (DP) if for any two predictions $\ftlp, \ftlp'\in [0, 1]$:
    \begin{equation*}
    Pr[\dpmech(\ftlp) \in \mathcal{I}]\leq e^{\dpe\cdot |\ftlp - \ftlp'|}\cdot \Pr[\dpmech(\ftlp') \in \mathcal{I}] + \dpd.
    \end{equation*}
\end{definition}

We construct our privacy-based algorithm by adding truncated noise, where truncation guarantees predictions fall in the range of $[0, 1]$. The truncation of noise $Y$ works in the following way: given a prediction $\ftlp$, for random variable $Y$ with unbounded support, we draw $X\sim D_\epsilon(\ftlp)$ such that 
\begin{equation*}
    \Pr[\ftlp + X = \ftplp] = \frac{\Pr[\ftlp + Y = \ftplp]}{\Pr[\ftlp + Y\in [0, 1]]}.
\end{equation*}

\subsection{Predictions for Decision-Making}

A decision maker faces a decision problem $(A, \statesp, U)$:
\begin{itemize}
    \item the decision maker (DM) selects an action $a\in A$;
    \item a payoff-relevant state $\state\in \statesp$ is realized;
    \item DM receives payoff $U:A\times\statesp\to \reals$.
\end{itemize}

When assisted with a prediction, the best response $a^*$ maps a prediction to the action:
\begin{equation}
    a^*(\pred) = \argmax_{a\in A}\expect{\state\sim\pred}{U(a, \state)}.
\end{equation}

When the DM best responds, she assumes the state is drawn from $\pred$ and takes the action that maximizes the expected payoff. We define the best-responding payoff as a function $\score$ of the prediction and the state.
\begin{definition}[Scoring Rule from Decision]
    Given a decision problem $U$, the scoring rule induced from $U$ and belief $\ftplp\in \Delta(\statesp)$ is
    \begin{equation*}
        \score_U(\ftplp, \state) = U(a^*(\ftplp), \state)
    \end{equation*}
\end{definition}

Proper scoring rules characterize scoring rules induced from a decision problem.
\begin{definition}[Proper Score]
    A scoring rule $\score:[0, 1]\times \{0, 1\}\to \reals$ is proper if and only if
    \begin{equation*}
        \expect{\state\sim \ftplp}{\score(\ftplp, \state)}\geq \expect{\state\sim \ftplp}{\score(\ftplp', \state)}, \forall \ftplp'\in [0, 1].
    \end{equation*}
\end{definition}

\Cref{claim: proper score equal payoff} shows that the space of best-responding payoff is equivalent to the space of proper scoring rules. Throughout the paper, we will write the best-responding decision payoff as proper scoring rules. 

\begin{claim}[\citealt{kleinberg2023u, hu2024predict}]
\label{claim: proper score equal payoff}
    There exists a bijective mapping between a bounded proper scoring rules and scoring rule induced from a decision problem with bounded payoff:
    \begin{itemize}
        \item Given a decision problem $U(\cdot, \cdot)\in [0, 1]$, the induced scoring rule $\score_U(\cdot, \cdot)\in [0, 1]$ is a proper scoring rule.
        \item Given a proper scoring rule $\score(\cdot, \cdot)\in [0, 1]$, there exists a decision problem $U(\cdot, \cdot)\in [0, 1]$ that induces $\score$. 
    \end{itemize}
\end{claim}

Given a set of reference predictors $\mathcal{B}$ and a set of proper scoring rules $\mathcal{S}$, we define the decision loss with regard to the set of reference predictors. 

\begin{definition}[Decision Loss]
    Given a set of reference predictors $\mathcal{B}$ and a set of proper scoring rules $\mathcal{S}$, the decision loss of a predictor $\ftplpred$ is 
    \begin{align*}
        \ploss(\ftplpred; \mathcal{B}) = \max_{\score\in \mathcal{S}, \btlpred\in \mathcal{B}} \expect{\ftplp, \btlp, \state\sim D_{\ftplpred, \btlpred, \statesp}}{\score(\btlp, \state) - \score(\ftplp, \state)}.
    \end{align*}
\end{definition}
Throughout the paper, we consider decision loss with regard to the set of all bounded proper scoring rules $\mathcal{S} = \{\score(\cdot, \cdot)\in [0, 1]\}$, i.e.\ all decision problems with bounded payoff. 

Our decision loss is closely related to omniprediction \citep{omni}. A predictor with $\epsilon$ decision loss is an $\epsilon$ omnipredictor with regard to reference predictors in $\mathcal{B}$ and the set of scoring rules in $\mathcal{S}$. 

\begin{definition}[Omniprediction]
\label{def: omniprediction}
    Given a set of reference predictors $\mathcal{B}$ and a set of proper scoring rules $\mathcal{S}$, a predictor is an $\epsilon$-omnipredictor with regard to $\mathcal{B}$ and $\mathcal{S}$ if
    \begin{align*}
        \expect{(\ftplp, \state)\sim D_{\ftplpred, \statesp}}{\score(\ftplp, \state)}\geq \expect{(\btlp,\state)\sim D_{\btlpred,\statesp}}{\score(\btlp, \state)} - \epsilon, \qquad&\forall \btlpred\in \mathcal{B}, \score\in \mathcal{S}.
    \end{align*}
\end{definition}

\subsection{Measures of Calibration Error}
\label{sec: prelim calibration error}
In this section, we define different calibration error metrics that are relevant to the paper. The definitions of error metrics follow the definitions of perfect calibration. We denote the Bayesian posterior of prediction values as $\hat{\ftplp} = \Pr[\state = 1 | \ftplpred = \ftplp]$.

\begin{definition}
    [Perfect Calibration]
    A predictor $\ftplpred$ is perfectly calibrated if $\ftplp = \hat{\ftplp}$ for any $\ftplp\in [0, 1]$. 
\end{definition}

We introduce relevant calibration errors to the paper by two categories: calibration error for decision-making and calibration error for machine learning. 

\subsubsection{Calibration Errors for Decision-Making}

The canonical calibration error metric is $\ECE$, the averaged bias in predictions.

\begin{definition}
    [Expected Calibration Error, $\ECE$]
    Given predictor $\ftplpred$, the expected calibration error is 
    \begin{equation*}
        \ECE(\ftplpred) = \expect{\ftplp\sim \ftplpred}{\big|\ftplp - \hat{\ftplp}\big|}.
    \end{equation*}
\end{definition}

The swap regret of a decision-maker is closely related to predictions being calibrated. Swap regret minimizers are special cases of omnipredictors where the set of reference predictors $\mathcal{B}$ is the set of post-processed predictors, i.e.\ by applying a mapping $\sigma: [0, 1]\to [0, 1]$ to the orginal predictions $\ftplp$. 
\begin{definition}
    [Swap Regret]
    Given a decision problem with proper scoring rule $\score$, a predictor $\ftplpred$, the swap regret for the decision-maker is 
    \begin{equation*}
        \swapS_{\score}(\ftplpred) = \max_{\sigma: [0, 1]\to [0, 1]}\expect{(\pred, \state)\sim D_{\pred, \state}}{\score\left(\sigma(\pred), \state\right) - \score(\pred, \state)}.
    \end{equation*}
\end{definition}

\Cref{prop: swap regret equals cfdl} shows that the swap regret equals the payoff improvement from calibrating a predictor. 

\begin{proposition}
\label{prop: swap regret equals cfdl}
    Given a decision problem with proper scoring rule $\score$, a predictor $\ftplpred$, the mapping $\sigma^*(\ftplp) = \Pr[\state|\ftplp]$ is the swap regret maximizing mapping, i.e.\ 
    \begin{equation*}
        \sigma^* \in \argmax_{\sigma: [0, 1]\to [0, 1]}\expect{(\pred, \state)\sim D_{\pred, \state}}{\score\left(\sigma(\pred), \state\right) - \score(\pred, \state)}.
    \end{equation*}the swap regret equals the payoff improvement from calibrating the predictor: $\sigma^*(\ftplp) = \Pr[\state = 1 | \ftplp]$.
\end{proposition}

\Cref{prop: no swap regret equals calibration} thus follows by the definition of calibration. 

\begin{proposition}[Foster and Vohra,  \citealp{foster1998asymptotic}]
\label{prop: no swap regret equals calibration}
    A predictor is calibrated if and only if for any decision problem $U$, the decision-maker has no swap regret. 
\end{proposition}

Motivated by \Cref{prop: no swap regret equals calibration}, \citet{hu2024predict} define Calibration Decision Loss ($\cdl$), the worst-case decision loss induced by miscalibration, with worst-case over all bounded proper scoring rules. Instead of decision loss that compares with a set of fixed reference predictor, $\cdl$ calculates the decision loss where the reference is the calibrated correspondence of the predictor to be evaluated. 

\begin{definition}[Calibration Decision Loss, $\cdl$]
    For predictor $\ftplpred$, the Calibration Decision Loss is defined as 
    \begin{equation*}
        \cdl(\ftplpred) = \max_{\text{proper }\score(\cdot, \cdot)\in [0, 1]}\swapS_\score(\ftplpred),
    \end{equation*}
    where the maximum is taken over all bounded proper scoring rules. 
\end{definition}

\citet{kleinberg2023u} shows that $\cdl$ is upperbounded by $\ECE$.

\begin{lemma}[\citealt{kleinberg2023u}]
\label{lem: cdl bounded by ece}
     For predictor $\ftplpred$, $\cdl$ is upperbounded by $\ECE$, i.e.\ 
$
        \cdl(\ftplpred) \leq \ECE(\ftplpred).
$
\end{lemma}

\subsubsection{Calibration Errors for Machine Learning}

The calibration errors in this section are continuos in prediction space. The calibration error metrics relevant to the paper are the smooth calibration error and the distance to calibration ($\distcal$).

\begin{definition}[Smooth Calibration Error]
    Given predictor $\ftplpred$, the smooth calibration error takes the supremum over the set $\Sigma$ of $1$-Lipschitz functions:
    \begin{equation*}
        \smooth(\ftplpred) = \sup_{\sigma\in \Sigma}\expect{(\ftplp, \state)\sim D_{\ftplpred, \statesp}}{\sigma(\ftplp)\cdot(\ftplp-\state)}
    \end{equation*}
\end{definition}
Note that without the constraint that $\sigma$ is $1$-Lipschitz, $\smooth$ is the same as $\ECE$. To see this, note that taking $\sigma = 1$ when $\ftplp - \hat{\ftplp}\geq 0$ and $\sigma=-1$ when $\ftplp-\hat{\ftplp}<0$ gives the same definition as $\ECE$.

Given a predictor $\ftplp$, the distance to calibration finds a calibrated predictor $\btlp$ with a coupling $D_{\ftplp, \btlp, \state}$ with the given predictor $\ftplp$, such that $\btlp$ has the smallest distance from $\ftplp$. The distance is the $\ell_1$ distance between predictions, as defined in \Cref{def: distance between predictors}.
\begin{definition}[Distance between Predictors]
\label{def: distance between predictors}
    Given predictors $\btlpred$ and $\ftplpred$, the distance between the predictors is defined as
    \begin{equation*}
        \distance(\ftplpred, \btlpred) = \expect{D_{\ftplpred, \btlpred}}{|\ftplpred - \btlpred|}.
    \end{equation*}
\end{definition}

\Cref{def: dtc} defines the distance to calibration. 

\begin{definition}[Distance to Calibration, $\distcal$]
\label{def: dtc}
    For predictor $\ftplpred$, the distance to calibration is
    \begin{equation*}
        \distcal(\ftplpred) = \min_{\btlpred \text{ is calibrated}}\distance(\ftplpred, \btlpred).
    \end{equation*}
\end{definition}

The smooth calibration error is linearly related to the distance to calibration. While in our paper, we mainly focus on $\distcal$, our results also apply to the smooth calibration error $\smooth$. 

\begin{lemma}[\citet{utc}]
    Given any predictor $\ftplpred$,
    \begin{equation*}
        \frac{1}{2}\distcal(\ftplpred)\leq \smooth(\ftplpred) \leq \distcal(\ftplpred).
    \end{equation*}
\end{lemma}

\subsection{Online and Batch Post-Processing Algorithm}

We design a post-processing algorithm for both the online setting and the batch setting, given predictions $\ftlp_1\dots \ftlp_T$ from a predictor $\ftlpred$. The post-processing algorithm knows the parameter $\distcal(\ftlpred) = \epsilon$. 

\paragraph*{The Online Setting} In the online setting, the goal of a post-processing algorithm is to generate trustworthy predictions $\ftplvec = (\ftplp_1, \dots, \ftplp_T)$ with low $\ECE$ or $\cdl$ given a sequence of predictions with low $\distcal$. At the end of $T$ rounds, the predictor is evaluated by a calibration error against the sequence of states $\vstate = (\state_1, \dots, \state_T)$. We define the joint distribution of $D_{\ftplpred, \statesp}$ in definitions in \Cref{sec: prelim calibration error} as the empirical distribution of $(\ftplp_t, \state_t)$ over $T$ rounds, which gives equivalent definitions of online calibration errors in the literature. We will write the calibration error of online predictors as a function of $\vpred$ and $\vstate$. 

In round $t\in [T]$, the adversary selects a prediction $\ftlp_t$. The post-processing algorithm $f = (f_t)_{t\in [T]}$ makes a (randomized) prediction according to $\alg_t$ given $\ftlp_t$ and the history of $(\ftlp_k, \ftplp_k)_{k\in [t-1]}$ but not the states\footnote{The algorithm in our paper only depends on $\ftlp_t$. This dependence on history only reinforces the definition.}. The adversary then reveals the state $\state_t$. The adversary knows the full history of interactions, i.e.\ $(\ftlp_k, \ftplp_k, \state_k)_{k\in [t-1]}$. When selecting the prediction $\ftlp_t$, the adversary faces the constraint that $\distcal(\ftlpred) = \epsilon$ at the end of $T$ rounds. 

Note that the restriction of the algorithm not knowing the state is slightly different from the classic online calibration \citep{foster1998asymptotic}. This restriction effectively excludes a post-processing algorithm that ignores the predictions $\ftlp$ and directly implements a calibrated predictor. 

\paragraph*{The Batch Setting} In the batch setting, the predictor $\ftlpred$ is specified by the joint distribution $D_{\ftlpred, \statesp}$ as introduced in the beginning of \Cref{sec: prelim}. We write $\ftlpred^T$ as the joint distribution of $T$ independent and identical draws of predictions from $\ftlpred$.  Given $T$ realizations of predictions $\ftlvec = (\ftlp_1, \dots, \ftlp_T)\sim \ftlpred^T$, the post-processing algorithm $\alg: [0, 1]^T\to \Delta([0, 1]^T)$ outputs (randomized) predictions $\ftplvec = (\ftplp_1, \dots, \ftplp_T)$. Since $\alg$ is only allowed to depend on predictions $\ftlvec$ not the states, it is without loss to write $\alg_{\ftlvec}(\ftlp):[0, 1]\to \Delta([0, 1])$, assuming the output follows the same distribution fixing samples $\ftlvec$. Then the states $\vstate = (\state_1, \dots, \state_T)$ is realized. In addition to the calibration errors as defined in \Cref{sec: prelim calibration error}, the algorithm is evaluated by the performance for omniprediction as in \Cref{def: omniprediction}, where the set of reference predictors $\mathcal{B}$ is the set of predictors with low $\distcal$ to $\ftlpred$.

\section{Smoothed Predictions for the Batch Setting}

\label{sec: alg batch setting}

In this section, we will focus on post-processing in the batch setting where $\ftlp$ is stochastically generated. Given a prediction $\ftlp\sim \ftlpred$, our privacy-based post-processing algorithm simply adds noise to $\ftlp$. We write the resulting predictor as $\ftplpred$, with randomness from both $\ftlpred$ and the privacy-based algorithm $\dpmech$. Note that in the batch setting where predictions and states are stochastically drawn, the privacy-based post-processing algorithm optimizes for the expected error, where the expectation is taken with randomness from both the prediction, the state, and the post-processing algorithm.

\begin{itemize}
    \item \textbf{Input}: prediction $\ftlp\sim \ftlpred$, parameter $\epsilon$ such that $\distcal(\ftlpred)\leq \epsilon$, DP mechanism $\dpmech$.
    \item \textbf{Output}: Prediction $\ftplp\sim \dpmech(\ftlp)$
\end{itemize}




\Cref{thm: batch main} characterizes the decision loss of $\ftplpred$ with regard to all proper scoring rules and all predictors that are $\epsilon$ close to $\ftlpred$.

\begin{theorem}
\label{thm: batch main}
    Suppose mechanism $\dpmech$ is $(\dpe, \dpd)$-differentially private, then the output predictor $\ftplpred$ has at most $C$ decision loss with regard to all proper scoring rules $\mathcal{S}$ and the set of calibrated predictors $\mathcal{B}$ such that any $\btlpred\in \mathcal{B}$ is $\epsilon$-close to $\ftlpred$, i.e.\ $\distance(\ftlpred, \btlpred) \leq \epsilon$. The bound $C$ is the following
    \begin{equation*}
        C \leq 2\max_{\ftlp\in [0, 1]}\expect{}{|\dpmech(\ftlp) - \ftlp|} + 4 \left(1-e^{-\dpe\epsilon}+\dpd \right).
    \end{equation*}

Moreover, $\ECE$ of $\ftplpred$ has the same bound.
\end{theorem}

 We prove \Cref{thm: batch main} following the idea of the Follow-The-Perturbed-Leader Algorithm \citep{kalai2005efficient}. We apply the same privacy-based post-processing algorithm $\dpmech$ to any calibrated predictor $\btlpred$ that is $\epsilon$ close to $\ftlpred$, which constructs a hypothetical predictor $\btplpred$ as an intermediate connecting $\btlpred$ and the post-processed predictor $\ftplpred = \dpmech(\ftlpred)$.  \Cref{thm: batch main} follows from combining \Cref{lem:R} and \Cref{lem: DP}, where \Cref{lem:R} bounds the decision loss from $\btlpred$ to $\btplpred$, and \Cref{lem: DP} characterizes the decision loss from $\btplpred$ to $\ftplpred$ via DP mechanism $\dpmech$.

\begin{lemma}\label{lem:R}
For any calibrated predictor $\btlpred$, we write $\btplpred$ as the resulting predictor with the post-privacy-based processing algorithm $\dpmech$ applied to $\btlpred$. For any bounded proper scoring rule $\score(\cdot, \cdot)\in [0, 1]$, the loss of $\btplpred$ is bounded, 
    \begin{equation*}
        \ploss(\btplpred)\leq 2\max_{\ftlp\in [0, 1]}\expect{}{|\dpmech(\ftlp) - \ftlp|}. 
    \end{equation*}
  The same bound holds for $\ECE$.
    \begin{equation*}
        \ECE(\btplpred)\leq \max_{\ftlp\in [0, 1]}\expect{}{|\dpmech(\ftlp) - \ftlp|}.
    \end{equation*}
\end{lemma}

\begin{lemma}\label{lem: DP}
           Suppose mechanism $\dpmech$ satisfies $(\dpe, \dpd)$-differentially privacy. We write $\btplpred$ as the resulting predictor with the privacy-based post-processing algorithm applied to calibrated predictor $\btlpred$ with $\distance(\ftlpred, \btlpred)\leq \epsilon$. The decision loss from $\btplpred$ to $\ftplpred$ is bounded by
            \begin{equation*}
                \expect{(\ftplp, \state)\sim D_{\ftplpred, \statesp}}{\score(\ftplp, \state)}\geq \expect{(\btplp, \state)\sim D_{\btplpred, \statesp}}{\score(\btplp, \state)} - 4\left(1-e^{-\dpe\epsilon}+\dpd \right).
            \end{equation*}
    
    A similar bound holds for $\ECE$:
    \begin{equation*}
        \ECE\left(\ftplpred\right)\leq \ECE\left(\btplpred\right)+4 \left(1-e^{-\dpe\epsilon}+\dpd \right).
    \end{equation*}
\end{lemma}

\Cref{lem: choice of noise} shows the guarantee obtainable from some choices of the differentially private mechanism by adding noise $D_{\epsilon}$. We construct the noise by truncating the distribution with unbounded support into the feasible range of predictions. The parameters of $(\gamma, \delta)$ are standard for Laplace and Gaussian noise \citep{dwork2014algorithmic}.

\begin{lemma}
\label{lem: choice of noise}
We consider two truncated noises that induce differential privacy. 

\begin{description}
    \item[Truncated Laplace] Noise variable $X$ from a truncated Laplace distribution with parameters $(0,-\frac{1}{\ln \rate})$ is $(-2\ln \rate ,0)$-differentially private. The expectation of the bias induced by noise is bounded: $\expect{}{|X|}\leq -\frac{1}{\ln \rate}-\frac{\rate}{1-\rate}$. Combining the bounds and taking $\rate = \exp\left(-\sqrt{\frac{1}{2\epsilon}}\right)$, we have $C=\Theta\left(\sqrt{\epsilon}\right)$, the decision loss of the predictor is bounded by $C$, and $\ECE\le C$.
    \item [Truncated Gaussian] Consider the truncated noise from a Gaussian distribution $\mathcal{N}\left(0, {2\epsilon\ln(\frac{1.25}{\sqrt{\epsilon}})}\right)$. The truncated noise has 
    \begin{equation*}
        \expect{}{|X|}\leq \sigma = \sqrt{2\epsilon\ln(\frac{1.25}{\sqrt{\epsilon}})},
    \end{equation*}
    and is $(\dpe, \dpd)$-differentially private with $\dpd = \sqrt{\epsilon}$ and 
   $ 
        1-e^{-\dpe\epsilon} \leq 2\sqrt{\epsilon}
    $
. Combining the bounds and taking $C = \Theta(\sqrt{\epsilon\ln(\frac{1}{{\epsilon}})})$, the decision loss of the predictor is bounded by $C$, and $\ECE\leq C$.\footnote{\Cref{appdx: improved for gaussian} shows an improved $O(\sqrt{\epsilon})$ bound for truncated Gaussian noise without the log factor. Note that we obtain \Cref{lem: choice of noise} by bounding the TV-distance between the DP-mechanism output of adjacent predictions. Our improved bound directly analyzes this TV-distance rather than using the $(\gamma, \delta)$ parameters of differential privacy. }
\end{description}

\end{lemma}

\Cref{prop: decision loss tight} shows that, there exists a predictor with $\distcal = \epsilon$, such that no post-processing algorithm can achieve a worst-case decision loss better than $\frac{\sqrt{\epsilon}}{2}$. Our guarantee of decision loss in \Cref{thm: batch main} is asymptotically optimal. 

\begin{theorem}[Post-processing Lowerbound for Batch Decision Loss]\label{prop: decision loss tight}
    There exists a predictor $\ftlpred$, with $\distcal(\ftlpred) = \epsilon$ and a reference calibrated predictor $\btlpred \in  \argmin_{\btlpred'} \distance(\btlpred', \ftlpred)$, such that for any post-processing algorithm that depends on the sequence $\ftlvec$ of predictions, $\alg_{\ftlvec}(\ftlp):[0, 1]\to \Delta([0, 1])$, $\alg_{\ftlvec}(\ftlpred)$ suffers a $\frac{\sqrt{\epsilon}}{2}$ decision loss from $\btlpred$, i.e.\ 
       \begin{equation*}
        \forall \alg, \exists\score(\cdot, \cdot)\in [0, 1], \quad\expect{\alg, (\ftplp, \state)\sim D_{\ftplpred, \statesp}}{\score(\alg_{\ftlvec}(\ftlp), \state)}\leq \expect{(\btlp, \state)\sim D_{\btlpred, \statesp}}{\score(\btlp, \state)} - \frac{\sqrt{\epsilon}}{2}.
       \end{equation*}
\end{theorem}

As the main idea of the proof of lowerbound, any post-processing algorithm that does not depend on the state achieves a score at most by outputting the Bayesian posterior of predictor $\ftlpred$. We construct a predictor $\ftlpred$ with a calibrated reference predictor $\btlpred$ that is more informative than $\ftlpred$. By definition of $\distcal$ that specifies a coupling between $\btlpred$, $\ftlpred$, and the state $\state$, a reference calibrated predictor $\btlpred$ may correlate with the state $\state$ when conditioned on $\ftlpred$. Thus, for this predictor $\ftlpred$ and any post-processing algorithm $\alg$, $\alg(\ftlpred)$ achieves a lower score than $\btlpred$. 

If the post-processing algorithm is a function of only the prediction $\ftlpred$ but not the prediction sequence $\ftlvec$, our bounds are asymptotically optimal. 

\begin{corollary}
\label{corollary: batch ece cdl bound}
       For any post-processing algorithm that depends only on the current prediction, $\alg(\ftlp):[0, 1]\to \Delta([0, 1])$, there exists a predictor $\ftlpred$ with $\distcal(\ftlpred) = \epsilon$ and a reference calibrated predictor $\btlpred \in  \argmin_{\btlpred'} \distance(\btlpred', \ftlpred)$, such that $\alg(\ftlpred)$ has
       \begin{equation*}
       \ECE(\ftlpred) = \Theta(\sqrt{\epsilon}) \qquad\text{and}\qquad \cdl(\ftlpred) =\Theta(\sqrt{\epsilon}).
       \end{equation*}
\end{corollary}
\Cref{corollary: batch ece cdl bound} is a corollary from \Cref{thm: online lower bound} which we will introduce later.

\section{Smoothed Predictions for the Online Setting}
\label{sec: online alg}

To achieve guarantees for the online setting where the predictions $\ftlvec$ and the states $\vstate$ are adversarially selected, the algorithm outputs are discretized for the empirical distribution to be meaningful. We prove empirical guarantees of the post-processing algorithm. 

\begin{itemize}
    \item \textbf{Input}: predictions $\ftlp_t$, parameter $\epsilon$ such that $\distcal(\ftlvec,  \vstate)\leq \epsilon$, DP mechanism $\dpmech$.
    \item Discretize the space of predictions into $T^{\frac{1}{3}}$ prediction values in $\{\frac{i}{T^{\frac{1}{3}}} \mid i\in [\epsilon T]\}$.
    \item Draw $\ftplp' \sim \dpmech(\ftlp)$. 
    \item Find $i$ such that $\ftplp'\in [\frac{i}{T^{\frac{1}{3}}}, \frac{i+1}{T^{\frac{1}{3}}}]$. 
    \item \textbf{Output}: $\ftplp = \frac{i}{T^{\frac{1}{3}}}$. 
\end{itemize}

By \Cref{thm: batch main}, the online privacy-based post-processing algorithm achieves the same bound for $\ECE$ up to a discretization error.
\begin{theorem}
    \label{thm: online main}
 Suppose mechanism $\dpmech$ is $(\dpe, \dpd)$-differentially private. The output predictor $\ftplvec$ satisfies
 \begin{equation*}
     \expect{}{\ECE(\ftplvec; \vstate)} \leq \max_{\ftlp\in [0, 1]}\expect{}{|\dpmech(\ftlp) - \ftlp|} + 4 \left(1-e^{-\dpe\epsilon}+\dpd \right) + 2T^{-\frac{1}{3}}.
 \end{equation*}
\end{theorem}

By \Cref{lem: cdl bounded by ece}, the same bound holds for $\cdl$
\begin{corollary}
     Suppose mechanism $\dpmech$ is $(\dpe, \dpd)$-differentially private. The output predictor $\ftplvec$ satisfies
 \begin{equation*}
     \expect{}{\cdl(\ftplvec; \vstate)} \leq 2\max_{\ftlp\in [0, 1]}\expect{}{|\dpmech(\ftlp) - \ftlp|} + 8 \left(1-e^{-\dpe\epsilon}+\dpd \right) + 2T^{-\frac{1}{3}}.
 \end{equation*}
\end{corollary}

By \Cref{lem: choice of noise}, we obtain the guarantees for $\ECE$ and $\cdl$ in the online setting. 
\begin{lemma}
\label{lem: noise online ece}
    With truncated Laplace noise, the privacy-based post-processing algorithm for online calibration achieves $\cdl\leq 2\ECE = O(\sqrt{\epsilon}) + 2T^{-\frac{1}{3}}$. With truncated Gaussian noise, the privacy-based post-processing algorithm achieves $\cdl\leq 2\ECE = O(\sqrt{\epsilon\ln{\frac{1}{\epsilon}}}) + 2T^{-\frac{1}{3}}$.
\end{lemma}

\citet{arunachaleswaran2025elementary} provides an online $\distcal$ minimization algorithm that achieves $\distcal = O(\frac{1}{\sqrt{T}})$. Plugging into \Cref{lem: noise online ece}, the post-processing algorithm achieves $\ECE = O(T^{-\frac{1}{4}})$ with truncated Laplace noise and $\ECE = O(T^{-\frac{1}{4}}\ln{T})$ with truncated Gaussian noise.

\Cref{thm: online lower bound} shows that there exist two sequences of predictions $\ftlvec, \ftlvec'$ and corresponding state realizations, such that both sequence has $\distcal=\epsilon$. However, no post-processing algorithm can guarantee $\ECE<\Theta(\sqrt{\epsilon})$ or $\cdl<\Theta(\sqrt{\epsilon})$ for both sequences. \Cref{thm: online lower bound} shows the online post-processing algorithm is asymptotically optimal for $\ECE$ as well as for $\cdl$. 

\begin{theorem}[Post-processing Lowerbound for Online $\ECE$]
\label{thm: online lower bound}
    For any post-processing algorithm $\alg = (\alg_1, \dots, \alg_T)$ where $\alg_t$ depends on the prediction history $(\ftlp_1, \dots, \ftlp_{t})$ and $(\ftplp_1, \dots, \ftplp_{t-1})$ before round $t$, there exists two sequences of predictions  $\ftlvec$ and $\ftlvec'$ with states $\vstate$ and $\vstate'$, respectively, both satisfying $\distcal(\ftlvec)=\distcal(\ftlvec')=\epsilon$, such that  
    \begin{equation*}
        \max\left\{\expect{}{\ECE\left(\ftplvec; \vstate\right)}, \expect{}{\ECE\left(\ftplvec';\vstate' \right)}\right\}\ge \frac{1}{8}\sqrt{\epsilon}+\frac{1}{2}\epsilon = \Theta(\sqrt{\epsilon}),
    \end{equation*}
    where we write $\ftplvec, \ftplvec'$ as the output of the post-processing algorithm $\alg$ on $\ftlvec, \ftlvec'$, respectively. 

    Moreover, the same lowerbound holds for $\cdl$.
\end{theorem}

The lowerbound for $\cdl$ is perhaps surprising because \citet{hu2024predict} shows a $\Tilde{O}(\frac{1}{\sqrt{T}})$ optimal bound for $\cdl$, indicating $\ECE$ overestimates $\cdl$ when there exists a $\omega(\frac{1}{\sqrt{T}})$ lowerbound for $\ECE$ \citep{sidestep}. We expected the same observation for the post-processing bound, which turns out not to be true. Considering the $\epsilon = \Omega(T^{-\frac{2}{3}})$ lowerbound for $\distcal$ \citep{qiao-distance}, the post-processing bound of $O(\sqrt{\epsilon}) + 2T^{-\frac{1}{3}}$ is asymptotically optimal. 

As an immediate corollary of our proof, even if the decision-makers are allowed to use different post-processing algorithms such as the differentially private exponential mechanism \cite{mcsherry2007mechanism}, there exists a worst-case decision-maker with a swap regret of $\Theta(\sqrt{\epsilon})$. 
\begin{corollary}
 There exists one decision-maker with proper scoring rule $\score$ such that for any post-processing algorithm $\alg = (\alg_1, \dots, \alg_T)$ where $\alg_t$ depends on the prediction history $(\ftlp_1, \dots, \ftlp_{t})$ and $(\ftplp_1, \dots, \ftplp_{t-1})$ before round $t$, there exists two sequences of predictions  $\ftlvec$ and $\ftlvec'$ with states $\vstate$ and $\vstate'$, respectively, both satisfying $\distcal(\ftlvec)=\distcal(\ftlvec')=\epsilon$, such that  
    \begin{equation*}
        \max\left\{\expect{}{\swapS_{\score}\left(f(\ftplvec); \vstate\right)}, \expect{}{\swapS_{\score}\left(f(\ftplvec');\vstate' \right)}\right\}\ge \frac{1}{8}\sqrt{\epsilon}+\frac{1}{2}\epsilon = \Theta(\sqrt{\epsilon}),
    \end{equation*}
\end{corollary}


\section{Discussion}

Our lowerbound presents a gap in post-processing a predictor with a low distance to calibration from directly optimizing for calibration errors related to decision-making. However, in the examples we present, the conditional empirical frequencies are discontinuous in the prediction space, which does not match the discussion of machine learning predictors not distinguishing between small perturbations. One follow-up question is, are there properties of the predictor that, combined with a low distance to calibration, guarantee the predictor trustworthy for decision-making after post-processing?

\citet{cal-gap} provides an answer to the question above. When the bias $\hat{\ftlp} - \ftlp$ is $1$-Lipschitz continuous in the prediction $\ftlp$, it follows that 
\begin{equation*}
    \ECE(\ftlpred)\leq O(\sqrt{\distcal(\ftlpred)}),
\end{equation*}
and no post-processing algorithm is needed. This result, however, suggests the same problem as suggested by our lowerbound, that a given predictor with low $\distcal$ achieves a non-optimal $\ECE$ or $\cdl$ compared to optimizing for $\ECE$ or $\cdl$ directly in the online setting. Thus, it remains a question whether there exists a property of a predictor with low distance to calibration that guarantees an optimal $\ECE$ or $\cdl$ from post-processing.

\newpage

\bibliographystyle{econ}
\bibliography{ref}

@article{hu2024predict,
  title={Predict to Minimize Swap Regret for All Payoff-Bounded Tasks (Calibration Error for Decision Making)},
  author={Hu, Lunjia and Wu, Yifan},
  journal={65th IEEE Symposium on Foundations of Computer Science (FOCS)},
  year={2024}
}

@inproceedings{ma,
author = {Kim, Michael P. and Ghorbani, Amirata and Zou, James},
title = {Multiaccuracy: Black-Box Post-Processing for Fairness in Classification},
year = {2019},
isbn = {9781450363242},
publisher = {Association for Computing Machinery},
address = {New York, NY, USA},
url = {https://doi.org/10.1145/3306618.3314287},
doi = {10.1145/3306618.3314287},
booktitle = {Proceedings of the 2019 AAAI/ACM Conference on AI, Ethics, and Society},
pages = {247–254},
numpages = {8},
location = {Honolulu, HI, USA},
series = {AIES '19}
}

@InProceedings{mc,
  title = 	 {Multicalibration: Calibration for the ({C}omputationally-Identifiable) Masses},
  author =       {Hebert-Johnson, Ursula and Kim, Michael and Reingold, Omer and Rothblum, Guy},
  booktitle = 	 {Proceedings of the 35th International Conference on Machine Learning},
  pages = 	 {1939--1948},
  year = 	 {2018},
  editor = 	 {Dy, Jennifer and Krause, Andreas},
  volume = 	 {80},
  series = 	 {Proceedings of Machine Learning Research},
  month = 	 {10--15 Jul},
  publisher =    {PMLR},
  url = 	 {https://proceedings.mlr.press/v80/hebert-johnson18a.html}
}

@article{smooth,
title = {Deterministic calibration and Nash equilibrium},
journal = {Journal of Computer and System Sciences},
volume = {74},
number = {1},
pages = {115-130},
year = {2008},
note = {Learning Theory 2004},
issn = {0022-0000},
doi = {https://doi.org/10.1016/j.jcss.2007.04.017},
url = {https://www.sciencedirect.com/science/article/pii/S0022000007000633},
author = {Sham M. Kakade and Dean P. Foster}
}

@article{omni-regression,
  title={Omnipredictors for Regression and the Approximate Rank of Convex Functions},
  author={Gopalan, Parikshit and Okoroafor, Princewill and Raghavendra, Prasad and Shetty, Abhishek and Singhal, Mihir},
  journal={arXiv preprint arXiv:2401.14645},
  year={2024}
}

@InProceedings{oracle-omni,
year = {2024},
author = {Sumegha Garg and Christopher Jung and Omer Reingold and Aaron Roth},
title = {Oracle Efficient Online Multicalibration and Omniprediction},
booktitle = {Proceedings of the 2024 Annual ACM-SIAM Symposium on Discrete Algorithms (SODA)},
chapter = {},
pages = {2725-2792},
doi = {10.1137/1.9781611977912.98},
URL = {https://epubs.siam.org/doi/abs/10.1137/1.9781611977912.98},
eprint = {https://epubs.siam.org/doi/pdf/10.1137/1.9781611977912.98}
}

@InProceedings{constrained,
  title = 	 {Omnipredictors for Constrained Optimization},
  author =       {Hu, Lunjia and Livni Navon, Inbal Rachel and Reingold, Omer and Yang, Chutong},
  booktitle = 	 {Proceedings of the 40th International Conference on Machine Learning},
  pages = 	 {13497--13527},
  year = 	 {2023},
  editor = 	 {Krause, Andreas and Brunskill, Emma and Cho, Kyunghyun and Engelhardt, Barbara and Sabato, Sivan and Scarlett, Jonathan},
  volume = 	 {202},
  series = 	 {Proceedings of Machine Learning Research},
  month = 	 {23--29 Jul},
  publisher =    {PMLR},
  url = 	 {https://proceedings.mlr.press/v202/hu23b.html}
}

@inproceedings{characterize-omni,
 author = {Gopalan, Parikshit and Kim, Michael and Reingold, Omer},
 booktitle = {Advances in Neural Information Processing Systems},
 editor = {A. Oh and T. Neumann and A. Globerson and K. Saenko and M. Hardt and S. Levine},
 pages = {39936--39956},
 publisher = {Curran Associates, Inc.},
 title = {Swap Agnostic Learning, or Characterizing Omniprediction via Multicalibration},
 url = {https://proceedings.neurips.cc/paper_files/paper/2023/file/7d693203215325902ff9dbdd067a50ac-Paper-Conference.pdf},
 volume = {36},
 year = {2023}
}

@InProceedings{omni,
  author =	{Gopalan, Parikshit and Kalai, Adam Tauman and Reingold, Omer and Sharan, Vatsal and Wieder, Udi},
  title =	{{Omnipredictors}},
  booktitle =	{13th Innovations in Theoretical Computer Science Conference (ITCS 2022)},
  pages =	{79:1--79:21},
  series =	{Leibniz International Proceedings in Informatics (LIPIcs)},
  ISBN =	{978-3-95977-217-4},
  ISSN =	{1868-8969},
  year =	{2022},
  volume =	{215},
  editor =	{Braverman, Mark},
  publisher =	{Schloss Dagstuhl -- Leibniz-Zentrum f{\"u}r Informatik},
  address =	{Dagstuhl, Germany},
  URL =		{https://drops-dev.dagstuhl.de/entities/document/10.4230/LIPIcs.ITCS.2022.79},
  URN =		{urn:nbn:de:0030-drops-156755},
  doi =		{10.4230/LIPIcs.ITCS.2022.79}
}

@article{qiao-distance,
  title={On the Distance from Calibration in Sequential Prediction},
  author={Qiao, Mingda and Zheng, Letian},
  journal={arXiv preprint arXiv:2402.07458},
  year={2024}
}

@inproceedings{cal-gap,
 author = {B{\l}asiok, Jaros{\l}aw and Gopalan, Parikshit and Hu, Lunjia and Nakkiran, Preetum},
 booktitle = {Advances in Neural Information Processing Systems},
 editor = {A. Oh and T. Neumann and A. Globerson and K. Saenko and M. Hardt and S. Levine},
 pages = {72071--72095},
 publisher = {Curran Associates, Inc.},
 title = {When Does Optimizing a Proper Loss Yield Calibration?},
 url = {https://proceedings.neurips.cc/paper_files/paper/2023/file/e4165c96702bac5f4962b70f3cf2f136-Paper-Conference.pdf},
 volume = {36},
 year = {2023}
}

@inproceedings{sidestep,
author = {Qiao, Mingda and Valiant, Gregory},
title = {Stronger calibration lower bounds via sidestepping},
year = {2021},
isbn = {9781450380539},
publisher = {Association for Computing Machinery},
address = {New York, NY, USA},
url = {https://doi.org/10.1145/3406325.3451050},
doi = {10.1145/3406325.3451050},
booktitle = {Proceedings of the 53rd Annual ACM SIGACT Symposium on Theory of Computing},
pages = {456–466},
numpages = {11},
location = {Virtual, Italy},
series = {STOC 2021}
}

@article{foster1998asymptotic,
  title={Asymptotic calibration},
  author={Foster, Dean P and Vohra, Rakesh V},
  journal={Biometrika},
  volume={85},
  number={2},
  pages={379--390},
  year={1998},
  publisher={Oxford University Press}
}

@inproceedings{utc,
  title={A unifying theory of distance from calibration},
  author={B{\l}asiok, Jaros{\l}aw and Gopalan, Parikshit and Hu, Lunjia and Nakkiran, Preetum},
  booktitle={Proceedings of the 55th Annual ACM Symposium on Theory of Computing},
  pages={1727--1740},
  year={2023}
}

@inproceedings{kleinberg2023u,
  title={U-calibration: Forecasting for an unknown agent},
  author={Kleinberg, Bobby and Leme, Renato Paes and Schneider, Jon and Teng, Yifeng},
  booktitle={The Thirty Sixth Annual Conference on Learning Theory},
  pages={5143--5145},
  year={2023},
  organization={PMLR}
}

@inproceedings{
noarov2023highdimensional,
title={High-Dimensional Unbiased Prediction for Sequential Decision Making},
author={Georgy Noarov and Ramya Ramalingam and Aaron Roth and Stephan Xie},
booktitle={OPT 2023: Optimization for Machine Learning},
year={2023},
url={https://openreview.net/forum?id=P4j4l45NUq}
}

@article{foster1997calibrated,
  title={Calibrated learning and correlated equilibrium},
  author={Foster, Dean P and Vohra, Rakesh V},
  journal={Games and Economic Behavior},
  volume={21},
  number={1-2},
  pages={40--55},
  year={1997},
  publisher={Academic Press}
}

@article{dagan2023external,
  title={From external to swap regret 2.0: An efficient reduction and oblivious adversary for large action spaces},
  author={Dagan, Yuval and Daskalakis, Constantinos and Fishelson, Maxwell and Golowich, Noah},
  journal={arXiv preprint arXiv:2310.19786},
  year={2023}
}

@article{dagan2024improved,
  title={Improved bounds for calibration via stronger sign preservation games},
  author={Dagan, Yuval and Daskalakis, Constantinos and Fishelson, Maxwell and Golowich, Noah and Kleinberg, Robert and Okoroafor, Princewill},
  journal={arXiv preprint arXiv:2406.13668},
  year={2024}
}

@article{foster2018smooth,
  title={Smooth calibration, leaky forecasts, finite recall, and nash dynamics},
  author={Foster, Dean P and Hart, Sergiu},
  journal={Games and Economic Behavior},
  volume={109},
  pages={271--293},
  year={2018},
  publisher={Elsevier}
}

@inproceedings{smoothECE,
  author={Jaroslaw Blasiok and Preetum Nakkiran},
  title={Smooth ECE: Principled Reliability Diagrams via Kernel Smoothing},
  year={2024},
  cdate={1704067200000},
  url={https://openreview.net/forum?id=XwiA1nDahv},
  booktitle={ICLR}
}

@article{kalai2005efficient,
  title={Efficient algorithms for online decision problems},
  author={Kalai, Adam and Vempala, Santosh},
  journal={Journal of Computer and System Sciences},
  volume={71},
  number={3},
  pages={291--307},
  year={2005},
  publisher={Elsevier}
}

@misc{dagan2024breakingt23barriersequential,
      title={Breaking the $T^{2/3}$ Barrier for Sequential Calibration}, 
      author={Yuval Dagan and Constantinos Daskalakis and Maxwell Fishelson and Noah Golowich and Robert Kleinberg and Princewill Okoroafor},  booktitle={the 56th Annual ACM Symposium on Theory of Computing (STOC 2025)},
  year={2025}
}

@inproceedings{
haghtalab2024truthfulness,
title={Truthfulness of Calibration Measures},
author={Nika Haghtalab and Mingda Qiao and Kunhe Yang and Eric Zhao},
booktitle={The Thirty-eighth Annual Conference on Neural Information Processing Systems},
year={2024}
}

@inproceedings{arunachaleswaran2025elementary,
  title={An Elementary Predictor Obtaining Distance to Calibration},
  author={Arunachaleswaran, Eshwar Ram and Collina, Natalie and Roth, Aaron and Shi, Mirah},
  booktitle={Proceedings of the 2025 Annual ACM-SIAM Symposium on Discrete Algorithms (SODA)},
  pages={1366--1370},
  year={2025},
  organization={SIAM}
}

@article{dwork2014algorithmic,
  title={The algorithmic foundations of differential privacy},
  author={Dwork, Cynthia and Roth, Aaron and others},
  journal={Foundations and Trends{\textregistered} in Theoretical Computer Science},
  volume={9},
  number={3--4},
  pages={211--407},
  year={2014},
  publisher={Now Publishers, Inc.}
}

@inproceedings{mcsherry2007mechanism,
  title={Mechanism design via differential privacy},
  author={McSherry, Frank and Talwar, Kunal},
  booktitle={48th Annual IEEE Symposium on Foundations of Computer Science (FOCS'07)},
  pages={94--103},
  year={2007},
  organization={IEEE}
}
\newpage
\appendix
\section{Missing Proof in \Cref{sec: alg batch setting}}

\subsection{Proof of \Cref{lem:R}}

\begin{proof}[Proof of \Cref{lem:R}]

Since $\score$ is bounded by $[0, 1]$, we know for any fixed $\btlp$, 
\begin{equation*}
    \expect{\state\sim \btlp}{\score(\btlp, \state) - \score(\btplp, \state)}\leq 2 |\btlp - \btplp|.
\end{equation*}

Thus, 
\begin{equation*}
    \expect{\btlp, \btplp}{\expect{\state\sim \btlp}{\score(\btlp, \state) - \score(\btplp, \state)}}\leq 2\expect{\btlp, \btplp}{|\btlp - \btplp|},
\end{equation*}
which proves the argument for decision loss. 

Now we prove \Cref{lem:R} for $\ECE$. We define $Y(b_0) = \dpmech(\btlp_0)-\btlp_0$. 
The joint probability density function of state $\state$ and prediction value $\btplp$ can be expressed as
\begin{align}
\Pr[\state =1 , \btplpred = \btplp] = &\int_{0}^{1}\Pr[\btlpred=\btlp]\cdot \Pr[\state = 1, \dpmech(\btlp) = \btplp|\btlpred=\btlp]\diff \btlp\nonumber\\
=&\int_{0}^{1}\Pr[\btlpred=\btlp]\cdot\Pr[\dpmech(\btlp) = \btplp]\cdot \Pr[\state = 1|\btlpred=\btlp]\diff \btlp\nonumber \\
\label{b_hat = b} = &\int_{0}^{1}\Pr[\btlpred=\btlp]\cdot\Pr[\dpmech(\btlp) = \btplp]\cdot \btlp\ \diff \btlp.
\end{align}
\Cref{b_hat = b} is derived given that $\btlpred$ is a calibrated predictor.

According to the definition of $\ECE$,
    \begin{align*}
        \ECE(\btplpred) &= \expect{\btplpred}{\big|\btplp - \Pr[\state=1 | \btplpred=\btplp]\big|} \\
        &= \int_{0}^1 \Pr[\btplpred=\btplp] \cdot\big|\btplp - \Pr[\state=1 | \btplpred=\btplp]\big|\diff \btplp \\
        &=\int_{0}^1 \big|\btplp\Pr[\btplpred=\btplp] - \Pr[\state = 1, \btplpred=\btplp]\big|\diff \btplp\\
        &= \int_{0}^{1}\left | \int_{0}^{1}\Pr[\dpmech(\btlp) = \btplp]\cdot \Pr[\btlpred=\btlp]\cdot \left ( \btplp - \btlp \right ) \mathrm{d}\btlp \right | \mathrm{d}\btplp\\
        &\leq  \int_{0}^{1} \int_{0}^{1}\Pr[\dpmech(\btlp) = \btplp]\cdot \Pr[\btlpred=\btlp]\cdot \left| \btplp - \btlp \right| \mathrm{d}\btlp \mathrm{d}\btplp\\
        &=\expect{}{\big|\dpmech(\btlp)-\btlp\big|}.
    \end{align*}
\end{proof}

\subsection{Proof of \Cref{lem: DP}}

\begin{proof}[Proof of \Cref{lem: DP}]
    We prove the lemma by \Cref{lem: dtv}, bounding the TV-distance between $\dpmech(\btlpred)$ and $\dpmech(\ftlpred)$. Combining with \Cref{lem:exp dtv dp}, we prove \Cref{lem: DP}.
\end{proof}
\begin{lemma}
\label{lem: dtv}
   We write $\btplpred$ as the resulting predictor with post-processing algorithm applied to calibrated predictor $\btlpred$ with $\distance(\ftlpred, \btlpred)\leq \epsilon$. The decision loss from $\btplpred$ to $\ftplpred$ is bounded by
            \begin{equation*}
                \expect{(\ftplp, \state)\sim D_{\ftplpred, \statesp}}{\score(\ftplp, \state)}\geq \expect{(\btplp, \state)\sim D_{\btplpred, \statesp}}{\score(\btplp, \state)} - 4\expect{(\btlp, \ftlp)\sim D_{\btlpred, \ftlpred}}{\dtv\left(\dpmech(\btlp), \dpmech(\ftlp)\right)}.
            \end{equation*}
    Note that the TV distance quantifies the distance between $\dpmech(\btlp) $ and $ \dpmech(\ftlp)$.
    
    A similar bound holds for $\ECE$:
    \begin{equation*}
        \ECE\left(\ftplpred\right)\leq \ECE\left(\btplpred\right)+4 \expect{(\btlp, \ftlp)\sim D_{\btlpred, \ftlpred}}{\dtv\left(\dpmech(\btlp), \dpmech(\ftlp)\right)}.
    \end{equation*}
\end{lemma}

\Cref{lem: dtv} follows from the fact that the scoring rule $\score$ is bounded in $[0, 1]$. 

\begin{proof}[Proof of \Cref{lem: dtv}] Since the scoring rule $\score$ is bounded in $[0, 1]$, we know for any fixed $\btlp$ and $\ftlp$, 
\begin{align*}
     &\expect{\btplp\sim \dpmech(\btlp),\state\sim \btlp}{\score(\btplp, \state)}-\expect{\ftplp\sim \dpmech(\ftlp),\state\sim \btlp}{\score(\ftplp, \state)}\\
     =&\int_0^1\left(\Pr[\dpmech(\btlp)=\ftplp]-\Pr[\dpmech(\ftlp)=\ftplp]\right)\expect{\state\sim\btlp}{\score(\ftplp,\state)}\mathrm{d}\ftplp\\
     \leq & 4\dtv\left(\dpmech(\btlp), \dpmech(\ftlp)\right).
\end{align*}
Thus \Cref{lem: dtv} for decision loss from $\btplpred$ to $\ftplpred$ holds.

Now we prove \Cref{lem: dtv} for $\ECE$ by dividing it into three parts, the first part is
\begin{equation} \label{R to P part 1}
    \begin{aligned}
    &\int_{0}^{1} \ftplp\big |\Pr[\ftplpred=\ftplp]- \Pr[\btplpred=\ftplp]\big|\mathrm{d}\ftplp\\
    =&\int_{0}^{1} \ftplp\left |\int_{0}^{1}\int_{0}^{1}\Pr[\btlpred=\btlp,\ftlpred=\ftlp]\left ( \Pr[\dpmech(\ftlp)=\ftplp]- \Pr[\dpmech(\btlp)=\ftplp]\right ) \mathrm{d}\btlp \mathrm{d}\ftlp\right|\mathrm{d}\ftplp\\
    \le&\int_{0}^{1}\int_{0}^{1}\Pr[\btlpred=\btlp,\ftlpred=\ftlp]\int_{0}^{1}\big| \Pr[\dpmech(\btlp)=\ftplp]-\Pr[\dpmech(\ftlp)=\ftplp] \big| \mathrm{d}\ftplp \mathrm{d}\btlp \mathrm{d}\ftlp\\
    =& 2\expect{(\btlp,\ftlp)\sim D_{\btlpred, \ftlpred}}{d_{\text{TV}}\left(\dpmech(\btlp), \dpmech(\ftlp)\right)}.
    \end{aligned}
\end{equation}
The distance between joint distribution $D_{\btplpred, \statesp}$ and $D_{\ftplpred, \statesp}$ is
\begin{equation} \label{R to P part 2}
\begin{aligned}
&\int_{0}^{1} \big |\Pr[\btplpred=\ftplp]\Pr\left [ \state=1 \mid \btplpred=\ftplp \right ] -\Pr[\ftplpred=\ftplp]\Pr\left [ \state=1 \mid \ftplpred=\ftplp \right ]\big|\mathrm{d}\ftplp\\
=&\int_{0}^{1} \big |\Pr\left[\state=1, \btplpred=\ftplp \right ] -\Pr\left[\state=1, \ftplpred=\ftplp \right ]\big|\mathrm{d}\ftplp\\
\le&\int_{0}^{1}\int_{0}^{1}\Pr\left[\btlpred=\btlp, \ftlpred=\ftlp\right]\int_{0}^{1}\big | \Pr[\dpmech(\btlp)=\ftplp]-\Pr[\dpmech(\ftlp)=\ftplp] \big | \mathrm{d}\ftplp \mathrm{d}\btlp \mathrm{d}\ftlp\\
=& 2\expect{(\btlp,\ftlp)\sim D_{\btlpred, \ftlpred}}{d_{\text{TV}}\left(\dpmech(\btlp), \dpmech(\ftlp)\right)}.
\end{aligned}
\end{equation}
Combine (\ref{R to P part 1}) and (\ref{R to P part 2}),
\begin{align*}
\ECE(\ftplpred)&=\int_{0}^{1} \Pr[\ftplpred=\ftplp]\big |\ftplp- \Pr\left [ \state=1 \mid \ftplpred=\ftplp \right ] \big |\mathrm{d}\ftplp\\
&\le \int_{0}^{1} \big  |\Pr[\btplpred=\ftplp]\left( \ftplp- \Pr\left [ \state=1 \mid \btplpred=\ftplp \right ] \right)\big |\mathrm{d}\ftplp\\
&+\int_{0}^{1} \ftplp\big  |\Pr[\ftplpred=\ftplp]- \Pr[\btplpred=\ftplp]\big |\mathrm{d}\ftplp\\
&+\int_{0}^{1} \big  |\Pr[\btplpred=\ftplp]\Pr\left [ \state =1\mid \btplpred=\ftplp \right ] -\Pr[\ftplpred=\ftplp]\Pr\left [ \state=1 \mid \ftplpred=\ftplp \right ]\big|\mathrm{d}\ftplp\\
&\leq \ECE\left(R\right)+4\expect{(\btlp,\ftlp)\sim D_{\btlpred, \ftlpred}}{d_{\text{TV}}\left(\dpmech(\btlp), \dpmech(\ftlp)\right)}.
\end{align*}
\end{proof}

\begin{lemma}
\label{lem:exp dtv dp}
   Given noise $X$, $Y$ for $(\dpe, \dpd)$-differential privacy, $X$ and $Y$ are drawn from the same distribution,
   \begin{equation*}
      \expect{(\btlp, \ftlp)\sim D_{\btlpred, \ftlpred}}{\dtv\left(\dpmech(\btlp), \dpmech(\ftlp)\right)}  \le 1-e^{-\dpe\epsilon}+\dpd .
   \end{equation*}
\end{lemma}

\begin{proof}[Proof of \Cref{lem:exp dtv dp}]
    For any pair of fixed $(\ftlp,\btlp)$, consider the set of prediction values $V=\{\ftplp\mid \pr{\dpmech(\btlp)=\ftplp}-\pr{\dpmech(\ftlp)=\ftplp}\le 0\}$. 

    By \Cref{def: differential privacy} of differential privacy,
    \begin{align*}
        \pr{\dpmech(\btlp)=\ftplp}-\pr{\dpmech(\ftlp)=\ftplp}&\ge e^{-\dpe \left|\btlp-\ftlp\right|}\left(\pr{\dpmech(\ftlp)=\ftplp}-\dpd\right)-\pr{\dpmech(\ftlp)=\ftplp}\\
        &=\left(e^{-\dpe \left|\btlp-\ftlp\right|}-1\right)\pr{\dpmech(\ftlp)=\ftplp}-\dpd e^{-\dpe \left|\btlp-\ftlp\right|}.
    \end{align*}
    Calculate $\dtv\left(\dpmech(\btlp), \ftlp + X\right)$ using prediction values in $V$:
    \begin{align*}
        \dtv\left(\dpmech(\btlp), \dpmech(\ftlp)\right)&=\int_{V}\left|\pr{\dpmech(\btlp)=\ftplp}-\pr{\dpmech(\ftlp)=\ftplp}\right|\mathrm{d}\ftplp\\
        &\le \int_{V} \left[\left(1-e^{-\dpe \left|\btlp-\ftlp\right|}\right)\pr{\dpmech(\ftlp)=\ftplp}+\dpd e^{-\dpe \left|\btlp-\ftlp\right|}\right]\mathrm{d}\ftplp\\
        &\le 1-e^{-\dpe \left|\btlp-\ftlp\right|}+\dpd e^{-\dpe \left|\btlp-\ftlp\right|}. 
    \end{align*}
    Take the expectation with respect to $(\btlp,\ftplp)$ and get
    \begin{align*}
        \expect{(\btlp, \ftlp)\sim D_{\btlpred, \ftlpred}}{\dtv\left(\dpmech(\btlp), \dpmech(\ftlp)\right)}&\le \expect{(\btlp, \ftlp)\sim D_{\btlpred, \ftlpred}}{1-\left(1-\dpd\right)e^{-\dpe \left|\btlp-\ftlp\right|}}\\
        &\le 1-\left(1-\dpd\right)e^{-\dpe\epsilon}\\
        &\le 1-e^{-\dpe\epsilon}+\dpd.
    \end{align*}
    The second inequality follows from Jensen's inequality, give that $1-\dpd\ge 0$, function $e^{-\dpe x}$ is convex and $\expect{(\btlp, \ftlp)\sim D_{\btlpred, \ftlpred}}{\left|\btlp-\ftplp\right|}=\epsilon$.
\end{proof}

Similarly combining \Cref{lem: DP} and \Cref{lem:R}, post-processed predictor $\ftplpred$ is calibrated in $\ECE$.

\subsection{Proof of \Cref{lem: choice of noise}}

\subsubsection{Truncated Laplace Noise}
\begin{proof}[Proof of \Cref{lem: choice of noise}, Truncated Laplace] For $\forall \ftlp,\ftplp\in [0,1]$ and differentially private mechanism $\dpmech$ as adding noise from the truncated Laplace distribution, $$\pr{\dpmech(\ftlp)=\ftplp}=\frac{-\ln \rate}{2-\rate^{\ftlp}-\rate^{1-\ftlp}}\cdot \rate^{|\ftplp-\ftlp|}.$$
\begin{align*}
    \frac{\Pr[\dpmech(\ftlp) = \ftplp]}{\Pr[\dpmech(\ftlp') = \ftplp]}= \rate^{|\ftplp-\ftlp|-|\ftplp-\ftlp'|}\cdot \frac{2-\rate^{\ftlp'}-\rate^{1-\ftlp'}}{2-\rate^{\ftlp}-\rate^{1-\ftlp}}.
\end{align*}
Since $|\ftplp-\ftlp|-|\ftplp-\ftlp'|\ge -|\ftlp-\ftlp'|$, \begin{equation*}
    \rate^{|\ftplp-\ftlp|-|\ftplp-\ftlp'|}\le \rate^{-|\ftlp-\ftlp'|}.
\end{equation*}
The following steps will show \begin{equation*}
    \frac{2-\rate^{\ftlp'}-\rate^{1-\ftlp'}}{2-\rate^{\ftlp}-\rate^{1-\ftlp}}\le  \rate^{-|\ftlp-\ftlp'|}.
\end{equation*}
Case 1: $\ftlp'\le \ftlp$. 
\begin{align*}
    &\frac{2-\rate^{\ftlp'}-\rate^{1-\ftlp'}}{2-\rate^{\ftlp}-\rate^{1-\ftlp}}\le  \rate^{-|\ftlp-\ftlp'|}\\
    \Leftrightarrow&-\rate^{\ftlp+1}\left(\rate^{-\ftlp'}\right)^2+2\rate^{\ftlp}\cdot \rate^{\ftlp'}+\rate^{1-\ftlp}\le 2.
\end{align*}
Since $\rate^{-\ftlp'}\in[1,\rate^{-\ftlp}]$, $-\rate^{\ftlp+1}\left(\rate^{-\ftlp'}\right)^2+2\rate^{\ftlp}\cdot \rate^{\ftlp'}+\rate^{1-\ftlp}$ achieves its maximum value at $\rate^{-\ftlp}$, and the maximum value is $2$.\\
Case 2: $\ftlp'\ge \ftlp$. 
\begin{align*}
    &\frac{2-\rate^{\ftlp'}-\rate^{1-\ftlp'}}{2-\rate^{\ftlp}-\rate^{1-\ftlp}}\le  \rate^{-|\ftlp-\ftlp'|}\\
    \Leftrightarrow&-\rate^{-\ftlp}\left(\rate^{\ftlp'}\right)^2+2\rate^{-\ftlp}\cdot \rate^{\ftlp'}+\rate^{\ftlp}\le 2.
\end{align*}
Since $\rate^{-\ftlp'}\in[\rate,\rate^{\ftlp}]$, $-\rate^{-\ftlp}\left(\rate^{\ftlp'}\right)^2+2\rate^{-\ftlp}\cdot \rate^{\ftlp'}+\rate^{\ftlp}$ achieves its maximum value at $\rate^{\ftlp}$, and the maximum value is $2$.

Therefore, \begin{equation*}
    \Pr[\dpmech(\ftlp) = \ftplp]\le \rate^{-2|\ftlp-\ftlp'|}\Pr[\dpmech(\ftlp') = \ftplp].
\end{equation*}
For any subset $\mathcal{I}\subseteq [0, 1]$ of predictions, \begin{equation*}
    Pr[\dpmech(\ftlp) \in \mathcal{I}]\leq \rate^{-2|\ftlp-\ftlp'|}\cdot \Pr[\dpmech(\ftlp')\in \mathcal{I}].
\end{equation*}
\end{proof}

\subsubsection{Truncated Gaussian Noise}
\begin{proof}[Proof of \Cref{lem: choice of noise}, Truncated Gaussian]
    The choice of parameters is adopted from \citet{dwork2014algorithmic}. We write the proof here for reference. The proof has two main steps. First, we show that the Gaussian distribution $Y\sim \mathcal{N}\left(0, {2\epsilon\ln(\frac{1.25}{\sqrt{\epsilon}})}\right)$ is $(\dpe_0, \dpd)$-differentially private with $\dpd = \sqrt{\epsilon}$ and 
   $ 
        1-e^{-\dpe_0\epsilon} \leq \sqrt{\epsilon}
    $. Then we show the probability that Gaussian is truncated is bounded by $1-\exp(-\frac{1}{4\sqrt{\epsilon}})$, implying 
    \begin{equation*}
        \frac{\Pr[X = p]}{\Pr[Y = p]}\leq \frac{1}{1-\exp\left(-\frac{1}{4\sqrt{\epsilon}}\right)}. 
    \end{equation*}
    By \Cref{def: differential privacy}, the truncated distribution has $\dpd = \sqrt{\epsilon}$ and $1-e^{-\dpe\epsilon}\leq 1-e^{-\dpe_0\epsilon}(1-\exp(-\frac{1}{4\sqrt{\epsilon}}))\leq 2\sqrt{\epsilon}$. 

    Now we show Gaussian distribution $Y\sim \mathcal{N}\left(0, {2\epsilon\ln(\frac{1.25}{\sqrt{\epsilon}})}\right)$ is differentially private. Notice that for \Cref{def: differential privacy}, it suffices to show 
        \begin{equation*}
       \Pr_{\ftplp\sim\ftlp+Y}\left[ \frac{\Pr_{Y}[\ftlp+Y = \ftplp]}{\Pr_{Y}[\ftlp'+Y = \ftplp]}\geq e^{\dpe_0 |\ftlp - \ftlp'|}\right]\leq \dpd.
    \end{equation*}
Define $L(\ftplp) = \frac{\Pr_{Y}[\ftlp+Y= \ftplp]}{\Pr_{Y}[\ftlp'+Y = \ftplp]}$. We know
\begin{align*}
    \ln [L(\ftplp)] = \frac{-(\ftplp - \ftlp)^2 + (\ftplp - \ftlp')^2}{2\sigma^2} = \frac{(\ftlp - \ftlp')^2 + 2(\ftplp - \ftlp)\cdot(\ftlp' - \ftlp')}{2\sigma^2},
\end{align*}
where $(\ftplp-\ftlp)$ is the Gaussian $\mathcal{N}(0, \sigma^2)$. Applying the tail bound for Gaussian distribution with $\dpe_1 = \dpe_0|\ftlp - \ftlp'|$
\begin{align*}
    \Pr[\ln[L(\ftplp)]\geq \dpe_1]\leq &\exp\left(-\frac{\left(\dpe_1\sigma^2 - \frac{1}{2}(\ftlp - \ftlp')^2\right)^2}{(\ftlp' - \ftlp)^2\sigma^2}\right)
\end{align*}

For $\sigma= \sqrt{2\epsilon\ln(\frac{1.25}{\dpd})}\geq \frac{\sqrt{2\ln(\frac{1.25}{\dpd})}\cdot |\ftlp - \ftlp'|}{\dpe_1}$, we have $ \Pr[\ln[L(\ftplp)]\leq\dpd$.

\end{proof}

\subsection{Improved Bound for Truncated Gaussian Noise}
\label{appdx: improved for gaussian}
For any distribution $D$ with probability density function $f$, define $f^{\btlp}(x)$ as the probability density function of truncated distribution of $D$ on the interval $[-\btlp,1-\btlp]$ and $f_b(x)=\frac{f(x)}{\int_{-\btlp}^{1-\btlp}f(x) \mathrm{d}x} $.
\begin{lemma}
\label{lem:general bound for dtv}
    Consider any distribution of noise with probability density function $f(x)$ that is monotone on $x\ge 0$ and $x<0$ respectively. Then for $\forall \btlp,\ftlp\in[0,1]$, $$d_{\text{TV}}\left(\dpmech(\btlp),\dpmech(\ftlp)\right)\le \max\{f^\btlp(x), f^\ftlp(x)\}\cdot \left|\ftlp-\btlp\right|.$$ 
\end{lemma}
\begin{proof}[Proof of \Cref{lem:general bound for dtv}]
    Fix $\btlp$ and $\ftlp$, without loss of generality, assume that $\btlp\le \ftlp$. There exists $t\in [\btlp,\ftlp]$ that $f^\btlp(t-\btlp)=f^\ftlp(t-\ftlp)$. 

    Represent the probability of $\dpmech(\btlp)\in[0,t]$ by $S=\int_{-\btlp}^{t-\btlp}f^\btlp(x)\mathrm{d}x $, so $d_{\text{TV}}\left(\dpmech(\btlp),\dpmech(\ftlp)\right)=S-\int_{-\ftlp}^{t-\ftlp}f^\ftlp(x)\mathrm{d}x$.
    
    \item When $t\ge \ftlp-\btlp$, represent the probability of $\dpmech(\btlp)\in[t+\btlp-\ftlp,t]$ by $S_1=\int_{t-\ftlp}^{t-\btlp}f^\btlp(x)\mathrm{d}x =S-\int_{-\btlp}^{t-\ftlp}f^\btlp(x)\mathrm{d}x$. The aim is to show $d_{\text{TV}}\left(\dpmech(\btlp),\dpmech(\ftlp)\right)\le S_1$.

    (i) If $\int_{-\btlp}^{1-\btlp}f(x) \mathrm{d}x\ge \int_{-\ftlp}^{1-\ftlp}f(x) \mathrm{d}x$, then 
    \begin{align*}
&d_{\text{TV}}\left(\dpmech(\btlp),\dpmech(\ftlp)\right)\le S_1\\
\Leftrightarrow &\int_{-\btlp}^{t-\ftlp}f^\btlp(x)\mathrm{d}x\le \int_{-\ftlp}^{t-\ftlp}f^\ftlp(x)\mathrm{d}x\\
\Leftrightarrow & \int_{-\btlp}^{t-\ftlp}\left(f^\btlp(x)-f^\ftlp(x)\right)\mathrm{d}x\le \int_{-\ftlp}^{-\btlp}f^\ftlp(x)\mathrm{d}x.\\
\end{align*}

    (ii) If $\int_{-\btlp}^{1-\btlp}f(x) \mathrm{d}x < \int_{-\ftlp}^{1-\ftlp}f(x) \mathrm{d}x$, then $\int_{t-\ftlp}^{0} f_b(x)\mathrm{d}x > \int_{t-\ftlp}^{0} f_q(x)\mathrm{d}x$. Since $$\int_{-\btlp}^{0} f_b(x)\mathrm{d}x=\frac{\int_{-\btlp}^{0}f(x)\mathrm{d}x}{\int_{-\btlp}^{1-\btlp}f(x)\mathrm{d}x}=\frac{1}{1+\frac{\int_{0}^{1-\btlp}f(x)\mathrm{d}x}{\int_{-\btlp}^{0}f(x)\mathrm{d}x} } $$ is an increasing function of $\btlp$, $$\int_{-\btlp}^{0} f_b(x)\mathrm{d}x\le \int_{-\ftlp}^{0} f_q(x)\mathrm{d}x.$$
    \begin{align*}
&d_{\text{TV}}\left(\dpmech(\btlp),\dpmech(\ftlp)\right)\le S_1\\
\Leftrightarrow &\int_{-\btlp}^{t-\ftlp}f^\btlp(x)\mathrm{d}x\le \int_{-\ftlp}^{t-\ftlp}f^\ftlp(x)\mathrm{d}x\\
\Leftrightarrow & \int_{-\btlp}^{0} f_b(x)\mathrm{d}x-\int_{t-\ftlp}^{0} f_b(x)\mathrm{d}x\le \int_{-\ftlp}^{0} f_q(x)\mathrm{d}x-\int_{t-\ftlp}^{0} f_q(x)\mathrm{d}x\\
\end{align*}
Therefore, $$d_{\text{TV}}\left(\dpmech(\btlp),\dpmech(\ftlp)\right)\le S_1\le (\ftlp-\btlp)\max\{f^\btlp(x), f^\ftlp(x)\}.$$

    \item When $t< \ftlp-\btlp$, $$d_{\text{TV}}\left(\dpmech(\btlp),\dpmech(\ftlp)\right)=S-\int_{-\ftlp}^{t-\ftlp}f^\ftlp(x)\mathrm{d}x<S<(\ftlp-\btlp)\max\{f^\btlp(x), f^\ftlp(x)\}.$$
\end{proof}

\begin{lemma}
    Consider adding the truncated noise from a Gaussian distribution $\mathcal{N}\left(0,\sqrt{\epsilon}\right)$ in the same way as \Cref{lem: choice of noise}, then for $C=\Theta(\sqrt{\epsilon})$, the predictor is $C$-omnipredictor with $\ECE\le C$.
\end{lemma}
\begin{proof}
    The truncated noise has 
    \begin{equation*}
        \expect{}{|X|}\leq \sigma = \sqrt{\epsilon}.
    \end{equation*}
    The maximum value of the truncated Gaussian distribution's probability density function is 
    \begin{equation*}
        \max_{\ftlp, \ftplp\in[0,1]} Pr_{\ftplp\sim \ftlp+X}[\ftlp+X=\ftplp]=\max_{\ftlp, \ftplp\in[0,1]}\frac{\frac{1}{\sqrt{2\pi}\sigma}\exp\left(-\frac{(\ftplp-\ftlp)^2}{2\sigma^2}\right)}{\int_{-\ftlp}^{1-\ftlp}\frac{1}{\sqrt{2\pi}\sigma}\exp\left(-\frac{x^2}{2\sigma^2}\right)\mathrm{d}x}=\frac{\frac{1}{\sqrt{2\pi}\sigma}}{\int_{0}^{1}\frac{1}{\sqrt{2\pi}\sigma}\exp\left(-\frac{x^2}{2\sigma^2}\right)\mathrm{d}x}.
    \end{equation*}
    Since $\exp\left(-\frac{x^2}{2\sigma^2}\right)$ is concave on $[0,\sigma]$, $\int_{0}^{\sigma}\frac{1}{\sqrt{2\pi}\sigma}\exp\left(-\frac{x^2}{2\sigma^2}\right)\mathrm{d}x$ can be lower bounded by the area of a ladder:
    \begin{equation*}
        \int_{0}^{1}\frac{1}{\sqrt{2\pi}\sigma}\exp\left(-\frac{x^2}{2\sigma^2}\right)\mathrm{d}x\ge \frac{1}{2\sqrt{2\pi}\sigma}\left(1+\exp\left(-\frac{1}{2}\right)\right)\cdot \sigma\ge\frac{1}{2\sqrt{2\pi}} .
    \end{equation*}
    By \Cref{lem:general bound for dtv},
    \begin{align*}
        \expect{(\btlp, \ftlp)\sim D_{\btlpred, \ftlpred}}{\dtv\left(\dpmech(\btlp), \dpmech(\ftlp)\right)}\le \expect{(\btlp, \ftlp)\sim D_{\btlpred, \ftlpred}}{\frac{2}{\sigma}\left|\ftlp-\btlp\right|}=\frac{2\epsilon}{\sigma}.
    \end{align*}
    Therefore, the parameter $C$ of the predictor can be upper bounded by $\sigma+\frac{8\epsilon}{\sigma}=\Theta(\sqrt{\epsilon})$.
\end{proof}

\subsection{Proof of \Cref{prop: decision loss tight}}

\begin{proof}[Proof of \Cref{prop: decision loss tight}]
    

    Fix a predictor $\ftlpred$, define predictor $\empiricalQ$ that predicts the Bayesian posterior of $\ftlpred$: for every prediction value $\ftlp_i$, when $\ftlpred$ predicts $\ftlp_i$, let $\empiricalQ$ predict $\empq{i}=\pr{\state=1\mid \ftlp = \ftlp_i}$. Post-process predictor $\ftlpred$ by $\alg$ and get predictor $\ftplpred$.

    Fix a prediction value $\ftlp_i$ and a proper scoring rule $\score$, consider all predictions $\ftplp\sim \alg(\ftlp_i)$, according to the definition of proper scoring rules, the score achievable by $\alg$ is upperbounded by $\Tilde{\ftlpred}$:  $$\expect{\ftplp\sim \alg(\ftlp_i)}{\expect{\state\sim \empq{i}}{\score(\ftplp,\state)}}\le \expect{\ftplp\sim \alg(\ftlp_i)}{\expect{\state\sim \empq{i}}{\score(\empq{i},\state)}}=\expect{\state\sim \empq{i}}{\score(\empq{i},\state)}.$$
    \begin{align*}
        \expect{(\ftplp,\state)\sim D_{\ftplpred, \statesp}}{\score(\ftplp,\state)}&=\expect{\ftlp_i\sim \ftlp}{\expect{\ftplp\sim \alg(\ftlp_i)}{\expect{\state\sim \empq{i}}{\score(\ftplp,\state)}}}\\
        &\le\expect{\ftlp_i\sim \ftlp}{\expect{\state\sim \empq{i}}{\score(\empq{i},\state)}}=\expect{(\ftplp,\state)\sim D_{\empiricalQ,\statesp}}{\score(\ftplp,\state)}.
    \end{align*}
    Consider the following predictor $\ftlpred$ with $\distcal(\ftlpred)=\epsilon$.
    
    Case 1: With probability $1-\sqrt{\epsilon}$, the distribution of predictions and states follows
    $$\left(\ftlp,\empq{}\right)=\left\{\begin{matrix} 
  (\frac{1}{2} -\sqrt{\epsilon },\frac{1}{2} -\sqrt{\epsilon })  & \mbox{w.p.}  \frac{1}{2} \\ 
  (\frac{1}{2} +\sqrt{\epsilon },\frac{1}{2} +\sqrt{\epsilon })& \mbox{w.p.} \frac{1}{2} 
\end{matrix}\right.  $$ 

Case 2: With probability $\sqrt{\epsilon}$, the distribution of predictions and states follows $$(\ftlp,\empq{})=\left\{\begin{matrix} 
  (\frac{1}{2} -\sqrt{\epsilon },1)  & \mbox{w.p.}  \frac{1}{2} \\ 
  (\frac{1}{2} +\sqrt{\epsilon },0)& \mbox{w.p.} \frac{1}{2} 
\end{matrix}\right.  $$
Therefore the corresponding $\empiricalq$ follows $$(\empiricalq,\ftlp)=\left\{\begin{matrix} 
  (\frac{1}{2} -\frac{1}{2} \sqrt{\epsilon }+\epsilon,\frac{1}{2} -\sqrt{\epsilon } )  & \mbox{w.p.}  \frac{1}{2} \\ 
  (\frac{1}{2} +\frac{1}{2} \sqrt{\epsilon }-\epsilon ,\frac{1}{2} +\sqrt{\epsilon })& \mbox{w.p.} \frac{1}{2} \\
\end{matrix}\right.$$

Define a calibrated predictor $\btlpred$, when $\ftlpred$ follows from Case 1, let $\btlpred$ outputs the same prediction of $\ftlpred$. When $\ftlpred$ follows from Case 2, let $\btlpred$ always predicts $\frac{1}{2}$. 
Notice that $$\distcal(\ftlpred)\le \distance(\ftlpred, \btlpred)= \epsilon,$$ to show $\distcal(\ftlpred)=\epsilon$, use a linear program with infinite constraints to prove $\distcal(\ftlpred)\ge\epsilon$.
Notice that $\ftldm=\{\frac{1}{2} -\sqrt{\epsilon }, \frac{1}{2} +\sqrt{\epsilon }\}$. 
Let $\density$ denotes joint probability distribution function of $(\btlp, \ftlp, \state)\in [0,1]\times \ftldm\times\{0,1\}$.
The following linear program is feasible and its optimal value equals $\distcal(\ftlpred)$.
\begin{align}
\label{DtC_LP}\mathrm{minimize}\quad &\sum _{(\btlp, \ftlp, \state)\in [0,1]\times \ftldm\times\{0,1\}} \left | \ftlp - \btlp \right | \density(\btlp,\ftlp,\state)\\
\mathrm{s.t.}\quad&\sum_{\btlp\in[0,1]}\density(\btlp,\ftlp,\state)=\pr{\ftlp,\state}, &&\text{for } \forall (\ftlp,\state)\in\ftldm\times\{0,1\};\hspace{0.5em}(r(\ftlp,\state))\nonumber\\
&(1-\btlp)\sum_{\ftlp\in\ftldm}\density(\btlp,\ftlp,1)-\btlp \sum_{\ftlp\in\ftldm}\density(\btlp,\ftlp,0)=0, && \text{for } \forall \btlp\in[0,1];\hspace{5.5em}(s(\btlp))\nonumber\\
& \density(\btlp,\ftlp,\state)\ge 0, && \text{for } \forall(\btlp, \ftlp, \state)\in [0,1]\times \ftldm\times\{0,1\}. \nonumber
\end{align} 
The objective of this linear program corresponds to $\distcal(\ftlpred)$. 
The first constraint ensures that the joint distribution of $(\btlp,\ftplp,\state)$ is consistent with the joint distribution of $(\ftlp,\state)$. The second constraint ensures that predictor $\btlpred$ is calibrated.
This linear program is feasible, because $$\density(\btlp, \ftlp,\state)=\left\{\begin{matrix} 
  \pr{\ftlp,\state}  & \mbox{if }\btlp=\state  \\ 
  0& \mbox{else}
\end{matrix}\right.  $$ is a feasible solution of this linear program.
The dual of the linear program (\ref{DtC_LP}) is
\begin{align}
    \label{DtC_LP_Dual}\mathrm{maximize} \quad&\sum_{(\ftlp,\state)\in \ftldm\times \{0,1\}}\pr{\ftlp,\state}r(\ftlp,\state)\\
    \mathrm{s.t.}\quad & r(\ftlp,\state)\le \left|\btlp-\ftlp\right|+(\state-\btlp)s(\btlp), && \text{for } \forall(\btlp, \ftlp, \state)\in [0,1]\times \ftldm\times\{0,1\}. \nonumber
\end{align}
If $s(\btlp)>1$, change $s(\btlp)$ to $1$ still satisfy the constraints and the objective stays the same:
\begin{align*}
    &r(\ftlp, 0)\le \left|\btlp - \ftlp\right|-\btlp s(\btlp)< \left|\btlp - \ftlp\right|-\btlp,\\
     &r(\ftlp, 1)\le \left|1 - \ftlp\right|\le\left|\btlp - \ftlp\right|+(1-\btlp).
\end{align*}
If $s(\btlp)<-1$, change $s(\btlp)$ to $-1$ still satisfy the constraints and the objective stays the same:
\begin{align*}
    &r(\ftlp, 0)\le \ftlp< \left|\btlp - \ftlp\right|+\btlp,\\
     &r(\ftlp, 1)\le \left|\btlp - \ftlp\right|+(1-\btlp) s(\ftlp)\le\left|\btlp - \ftlp\right|-(1-\btlp).
\end{align*}
Therefore, the optimal solution of linear program (\ref{DtC_LP_Dual}) stays the same after adding the constraints: $$-1\le s(\btlp)\le 1, \quad\text{for } \forall \btlp\in[0,1].$$
The optimal value of linear program (\ref{DtC_LP}) can be lower bounded by the objective of linear program (\ref{DtC_LP_Dual}):
\begin{align}
    &\sum_{(\ftlp,\state)\in \ftldm\times \{0,1\}}\pr{\ftlp,\state}r(\ftlp,\state)\nonumber\\
\label{convergent1}=&\sum_{(\ftlp,\state)\in\ftldm\times\{0,1\}}r(\ftlp,\state)\sum_{\btlp\in[0,1]}\density(\btlp,\ftlp,\state)+\sum_{\btlp\in[0,1]}s(\btlp)\sum_{(\ftlp,\state)\in\ftldm\times\{0,1\}}(\btlp-\state)\density(\btlp,\ftlp,\state)\\
\label{convergent2}=&\sum_{(\ftlp,\state)\in\ftldm\times\{0,1\}}\sum_{\btlp\in[0,1]}r(\ftlp,\state)\density(\btlp,\ftlp,\state)+\sum_{\btlp\in[0,1]}\sum_{(\ftlp,\state)\in\ftldm\times\{0,1\}}s(\btlp)(\btlp-\state)\density(\btlp,\ftlp,\state)\\
\label{convergent3}    =&\sum_{(\btlp, \ftlp,\state)\in [0,1]\times \ftldm\times \{0,1\}}\left[r(\ftlp,\state)+(\btlp-\state)s(\btlp)\right]\density(\btlp, \ftlp,\state)\\
    \le & \sum _{(\btlp, \ftlp, \state)\in [0,1]\times \ftldm\times\{0,1\}} \left | \ftlp - \btlp \right | \density(\btlp,\ftlp,\state).\nonumber
\end{align}
(\ref{convergent1})$=$(\ref{convergent2}) holds because $\sum_{\btlp\in[0,1]}\density(\btlp,\ftlp,\state)$ is absolutely convergent, the distributive property of multiplication still holds. (\ref{convergent2})$=$(\ref{convergent3}) holds because \Cref{convergent2} is absolutely convergent, the commutative property of addition still holds.

Let $$s(\btlp)=\left\{\begin{matrix} 
  \frac{2\sqrt{\epsilon } }{2\sqrt{\epsilon }+1 }  & \mbox{if }  \btlp< \frac{1}{2} \\ 
0  & \mbox{if }  \btlp= \frac{1}{2} \\ 
  -\frac{2\sqrt{\epsilon } }{2\sqrt{\epsilon }+1 }& \mbox{if } \btlp>\frac{1}{2} 
\end{matrix}\right.$$
Then the constraints for the dual linear program (\ref{DtC_LP_Dual}) are \begin{align*}
    &r\left(\frac{1}{2}-\sqrt[]{\epsilon },0 \right)\le \min_{\btlp\in[0,1]}\{\left | \btlp- \frac{1}{2}+\sqrt{\epsilon }\right |-\btlp s\left(\btlp\right) \}=\frac{-\sqrt{\epsilon }\left(1-2\sqrt{\epsilon }  \right) }{2\sqrt{\epsilon } +1}, \\
    &r\left(\frac{1}{2}-\sqrt[]{\epsilon },1 \right)\le \min_{\btlp\in[0,1]}\{\left | \btlp- \frac{1}{2}+\sqrt{\epsilon }\right |+\left(1-\btlp\right) s\left(\btlp\right) \}=\sqrt{\epsilon }, \\
    &r\left(\frac{1}{2}+\sqrt[]{\epsilon },0 \right)\le \min_{\btlp\in[0,1]}\{\left | \btlp- \frac{1}{2}-\sqrt{\epsilon }\right |-\btlp s\left(\btlp\right) \}=\sqrt{\epsilon }, \\
    &r\left(\frac{1}{2}+\sqrt{\epsilon },1 \right)\le \min_{\btlp\in[0,1]}\{\left | \btlp- \frac{1}{2}-\sqrt{\epsilon }\right |+\left(1-\btlp\right) s\left(\btlp\right) \}=\frac{-\sqrt{\epsilon }\left(1-2\sqrt{\epsilon }  \right) }{2\sqrt{\epsilon } +1}.
\end{align*}
Take maximum values of all $r\left(\ftlp,\state\right)$ and get the optimal value of linear program (\ref{DtC_LP_Dual}) is no less than 
\begin{align*}
    &\frac{1}{2}\left(\frac{1}{2}+\frac{\sqrt{\epsilon}}{2}-\epsilon\right)\left[r\left(\frac{1}{2}-\sqrt[]{\epsilon },0 \right)+r\left(\frac{1}{2}+\sqrt{\epsilon },1 \right)\right]\\+&\frac{1}{2}\left(\frac{1}{2}-\frac{\sqrt{\epsilon}}{2}+\epsilon\right)\left[r\left(\frac{1}{2}-\sqrt[]{\epsilon },1 \right)+r\left(\frac{1}{2}+\sqrt{\epsilon },0\right)\right]=\epsilon.
\end{align*}
Therefore, $\distcal(\ftlpred)\ge\epsilon$ and thus $\distcal(\ftlpred)=\epsilon$.

Consider the proper scoring rule
    $$\score(p,\state)=\left\{\begin{matrix} 
  1-\state & \mbox{if }  p\le \frac{1}{2} \\ 
  \state& \mbox{if } p>\frac{1}{2} 
\end{matrix}\right. $$ and calculate the expected payoff in decision making for predictor $\ftlpred$ and $\btlpred$:
$$\expect{\left(\ftplp,\state\right)\sim D_{\empiricalQ, \statesp}}{\score(\ftplp,\state)}=\frac{1}{2} +\frac{1}{2} \sqrt{\epsilon }-\epsilon.$$
$$\expect{(\btlp,\state)\sim  D_{\btlpred, \statesp}}{\score(\btlp,\state)}=\frac{1}{2} +\sqrt{\epsilon }-\epsilon.$$
Therefore, for any post-processed algorithm $\alg$, there exists predictor $\ftlpred$ and a reference calibrated predictor $\btlp$ such that $\distcal(\ftlpred)=\epsilon$ and $$\ploss(\alg(\ftlpred); \btlpred)\ge \expect{(\btlp,\state)\sim D_{\btlpred, \statesp}}{\score(\btlp,\state)}-\expect{(\ftplp,\state)\sim D_{\empiricalQ, \statesp}}{\score(\ftplp,\state)}=\frac{\sqrt{\epsilon}}{2}.$$
\end{proof}

\section{Missing Proof in \Cref{sec: online alg}}

\subsection{Proof of \Cref{thm: online main}}
\begin{proof}[Proof of \Cref{thm: online main}]
    We write $n_i$ as the number of times that $\frac{i}{T^{\frac{1}{3}}}$ is predicted. Clearly, $\sum_{i\in [\epsilon T]}n_i = T$. We also write $\ftplp'_t$ as the output of post-processed predictor before discretization. Conditioning on a set of $(n_i)_i$, we know for each $n_i$:
    \begin{align*}
        &\expect{}{\bigg|\frac{i}{T^{\frac{1}{3}}} - \sum_{t}\ind{\ftplp_t = \frac{i}{T^\frac{1}{3}}}\frac{\state_t}{n_i}\bigg|}\\
        \leq &\expect{}{\bigg|\frac{i}{T^{\frac{1}{3}}} - \sum_{t}\ind{\ftplp_t = \frac{i}{T^\frac{1}{3}}}\ftplp'_t\bigg|} + \frac{1}{n_i}\expect{}{\bigg|\sum_{t}\ind{\ftplp_t = \frac{i}{T^\frac{1}{3}}}\state_t - \sum_{t}\ind{\ftplp_t=\frac{i}{T^\frac{1}{3}}}\ftplp'_t\bigg|}\\
        \leq &T^{-\frac{1}{3}} + \sqrt{\variance{}{\sum_{t}\ind{\ftplp = \frac{i}{T^\frac{1}{3}}}\frac{\state_t}{n_i}\bigg|}} \\
        &+ \frac{1}{n_i}\expect{}{\sum_{t}\bigg|\ind{\ftplp_t = \frac{i}{T^\frac{1}{3}}}\Pr[\state | \ftplp_t'] - \sum_{t}\ind{\ftplp_t=\frac{i}{T^\frac{1}{3}}}\ftplp'_t\bigg|},
    \end{align*}
    where $\Pr[\state | \ftplp_t']$ is defined over the empirical distribution over $T$ rounds with the noise of the algorithm. Summing over all prediction values, we know 
    \begin{equation*}
        \frac{1}{T}\sum_i\expect{}{\bigg|\frac{i}{T^{\frac{1}{3}}} - \sum_{t}\ind{\ftplp = \frac{i}{\epsilon T}}\frac{\state_t}{n_i}\bigg|}\leq \sum_{i}\frac{1}{\sqrt{n_i}} + \ECE(\ftplpred) + T^{-\frac{1}{3}}\leq \ECE(\ftplpred) + 2T^{-\frac{1}{3}}. 
    \end{equation*}
    
\end{proof}

\subsection{Proof of \Cref{thm: online lower bound}}

 We restate our lemmas for $\ECE$ and $\CDL$ separately here and prove them. 

\begin{theorem}
\label{thm: online lower bound ece}
    For any post-processing algorithm $\alg$, there exists two sequences of predictions  $\ftlvec$ and $\ftlvec'$ with states $\vstate$ and $\vstate'$, respectively, both satisfying $\distcal(\ftlvec)=\distcal(\ftlvec')=\epsilon$, such that  
    \begin{equation*}
        \max\left\{\expect{}{\ECE\left(\ftplvec; \vstate\right)}, \expect{}{\ECE\left(\ftplvec';\vstate' \right)}\right\}\ge \frac{1}{8}\sqrt{\epsilon}+\frac{1}{2}\epsilon = \Theta(\sqrt{\epsilon}),
    \end{equation*}
    where we write $\ftplvec, \ftplvec'$ as the output of the post-processing algorithm $\alg$ on $\ftlvec, \ftlvec'$, respectively. 

\end{theorem}

\begin{lemma}
\label{ECE lower bound q1}
    Given predictor $\ftlpred = (\ftlp_1, \dots, \ftlp_T)$, and a post-processing algorithm $\alg = (\alg_1, \dots, \alg_T)$, suppose the empirical posterior for each prediction is $\hat{\ftlpred} = (\hat{\ftlp}_1, \dots_, \hat{\ftlp}_T)$. There exists a sequence of states $\vstate$ such that $\vstate$ is compatible with the empirical posterior, i.e.\ 
    \begin{equation*}
        \forall i\in [T], \hat{\ftlp}_i =  \frac{\sum_{t\in [T]}\state_t\ind{\ftlp_t = \ftlp_i}}{\sum_{t\in [T]}\ind{\ftlp_t = \ftlp_i}}.
    \end{equation*}
    Moreover, the expected $\ECE$ of the predictor $\alg$ with states $\vstate$ is lowerbounded
    \begin{equation*}
        \expect{\ftplvec\sim\alg}{\ECE(\ftplvec, \vstate)}\geq \expect{\ftplvec\sim \alg}{\frac{1}{T}\sum_{\ftplp\in \supp(\ftplvec)}\bigg|\sum_{t\in [T]}\left(\ftplp - \hat{\ftlp}_t\right)\cdot\ind{\ftplp_t = \ftplp}\bigg|},
    \end{equation*}
    where $\supp$ is the support of the output of $\alg$ in each round. 
\end{lemma}
\begin{proof}
Define $S_{\state}=\{\vstate\mid \vstate \text{ is compatible with the empirical posterior}\}$. Let $\vstate$ be chosen uniformly at random from $S_{\state}$, fix a sequence of predictions $\ftplvec$.

    Given the distribution of $\vstate$, $\expect{\vstate\in S_{\state}}{\sum_t\hat{\ftplp_i}\ind{\ftplp_t=\ftplp_i}}=\sum_t \hat{\ftlp_t}\ind{\ftplp_t=\ftplp_i}$ holds for any sequences of predictions $\ftplvec$ and any $i\in [T]$.
 By Jensen's Inequality, \begin{align*}
        \expect{\vstate\in S_{\state}}{\left|\sum_t\left(\ftplp_i-\hat{\ftplp_i}\right)\ind{\ftplp_t=\ftplp_i}\right|}&\ge \left|\sum_t\left(\ftplp_i\ind{\ftplp_t=\ftplp_i}-\expect{\vstate\in S_{\state}}{\hat{\ftplp_i}\ind{\ftplp_t=\ftplp_i}}\right)\right|\\
        &=\left|\sum_t\left(\ftplp_i-\hat{\ftlp_t}\right)\ind{\ftplp_t=\ftplp_i}\right|,
    \end{align*} apply this inequality to every prediction value $\ftplp_t$:
    \begin{align*}
        \expect{\vstate\in S_{\state}}{\ECE\left(\ftplvec\right)}&=\frac{1}{T}\sum_{\ftplp_i}\expect{\vstate\in S_{\state}}{\left|\sum_t\left(\ftplp_i-\hat{\ftplp_i}\right)\ind{\ftplp_t=\ftplp_i}\right|}\\
        &\ge\frac{1}{T} \sum_{\ftplp_i}\left|\sum_t\left(\ftplp_i-\hat{\ftlp_t}\right)\ind{\ftplp_t=\ftplp_i}\right|
    \end{align*} Take expectation on the distribution of predictions,
    \begin{equation*}
    \mathbf{E}_{\vstate\in S_{\state}}\expect{\ftplvec\sim \ftplpred}{\ECE\left(\ftplvec\right)}\ge\expect{\ftplvec\sim \alg}{\frac{1}{T}\sum_{\ftplp\in \supp(\ftplvec)}\bigg|\sum_{t\in [T]}\left(\ftplp - \hat{\ftlp}_t\right)\cdot\ind{\ftplp_t = \ftplp}\bigg|}.
    \end{equation*}
    Therefore, there must exist a sequence of states $\vstate$ that \begin{equation*}
            \expect{\ftplvec\sim \ftplpred}{\ECE\left(\ftplvec\right)} \geq \mathbf{E}_{\vstate\in S_{\state}}\expect{\ftplvec\sim \ftplpred}{\ECE\left(\ftplvec\right)}\geq \expect{\ftplvec\sim \alg}{\frac{1}{T}\sum_{\ftplp\in \supp(\ftplvec)}\bigg|\sum_{t\in [T]}\left(\ftplp - \hat{\ftlp}_t\right)\cdot\ind{\ftplp_t = \ftplp}\bigg|}.
    \end{equation*}
\end{proof}

\begin{proof} [Proof of \Cref{thm: online lower bound ece}]
    Assume there are $2T$ rounds, define $\ftlvec$ and $\ftlvec'$ as following: $$\ftlp_t=\left\{\begin{matrix} 
  \frac{1}{2}-\sqrt{\epsilon }   & \mbox{if }  t\le T \\ 
  \frac{1}{2}+\sqrt{\epsilon }& \mbox{else }
\end{matrix}\right. $$  $$\sum_{t=1}^{T}\mathbb{I}\left[\state_t=1\right]=T\left(\frac{1}{2}-\frac{1}{2}\sqrt{\epsilon}+\epsilon\right), \sum_{t=T+1}^{2T}\mathbb{I}\left[\state_t=1\right]=T\left(\frac{1}{2}+\frac{1}{2}\sqrt{\epsilon}-\epsilon\right).$$
For any $t\in[2T]$, $\ftlp_t'=\frac{1}{2}-\sqrt{\epsilon }$ and $\sum_{t=1}^{2T}\mathbb{I}\left[\state_t'=1\right]=2T\left(\frac{1}{2}-\sqrt{\epsilon}-\epsilon\right)$. 
Define $\hat{\ftlp}^0=\frac{1}{2}-\sqrt{\epsilon}-\epsilon$, $\hat{\ftlp}^1=\frac{1}{2}-\frac{1}{2}\sqrt{\epsilon}+\epsilon$, $\hat{\ftlp}^1=\frac{1}{2}+\frac{1}{2}\sqrt{\epsilon}-\epsilon$.

Fix a post-processing algorithm $\alg$. For any sequence of predictions $\ftplvec$ generated by post-processing $\ftlvec$, denote the distribution of $\ftplvec$ by $\mathbf{\alg}(\ftlvec)$.

For any $t'\in [2T]$ and any sequence of predictions $\ftplvec\sim \mathbf{\alg}(\ftlvec)$, define \begin{equation*} A(\ftplvec)_{t'}=\frac{\hat{\ftlp}^1\sum_{t=1}^T\ind{\ftplp_t=\ftplp_{t'}}+\hat{\ftlp}^2\sum_{t=T+1}^{2T}\ind{\ftplp_t=\ftplp_{t'}}}{\sum_{t=1}^{2T}\ind{\ftplp_t=\ftplp_{t'}}}\in[\hat{\ftlp}^1, \hat{\ftlp}^2].
\end{equation*}

According to \cref{ECE lower bound q1}, there exists a sequence of states $\vstate$ that \begin{align}
    \expect{\ftplvec\sim\mathbf{\alg}(\ftlvec) }{\ECE\left(\ftplvec\right)}&\ge \expect{\ftplvec\sim \mathbf{\alg}(\ftlvec)}{\frac{1}{2T}\sum_{\ftplp\in \supp(\mathbf{\alg}(\ftlvec))}\bigg|\sum_{t\in [2T]}\left(\ftplp - \hat{\ftlp}_t\right)\cdot\ind{\ftplp_t = \ftplp}\bigg|}\nonumber\\
\label{pt-At}    &=\expect{\ftplvec\sim \mathbf{\alg}(\ftlvec)}{\frac{1}{2T}\sum_{t\in[2T]}\left|\ftplp_t-A(\ftplvec)_t\right|}.
\end{align}
According to \cref{ECE lower bound q1}, \begin{equation} \label{pt-p0}
     \expect{\ftplvec\sim\mathbf{\alg}(\ftlvec') }{\ECE\left(\ftplvec\right)}\ge \expect{\ftplvec\sim \mathbf{\alg}(\ftlvec')}{\frac{1}{2T}\sum_{t\in[2T]}\left|\ftplp_t-\hat{\ftplp}^0\right|}.
\end{equation}
For any $\ftplvec\sim \mathbf{\alg}(\ftlvec)$ and $\ftplvec'\sim \mathbf{\alg}(\ftlvec')$, $\ftplp_t=\ftplp_{t}'$ always holds for $t\in[T]$, since $\ftlp_t=\ftlp_t'$ always holds for $t\in[T]$. Therefore, for any $t\in[T]$, \begin{equation*}
    \left|\ftplp_t-A(\ftplvec)_t\right|+\left|\ftplp_t'-\hat{\ftplp}^0\right|\ge \left|A(\ftplvec)_t-\hat{\ftplp}^0\right|\ge \frac{1}{2}\sqrt{\epsilon}+2\epsilon.
\end{equation*}
Add up inequality (\ref{pt-At}) and inequality (\ref{pt-p0}), 
\begin{align*}
     &\expect{\ftplvec'\sim\mathbf{\alg}(\ftlvec') }{\ECE\left(\ftplvec'\right)}+ \expect{\ftplvec\sim\mathbf{\alg}(\ftlvec) }{\ECE\left(\ftplvec\right)}\\
     \ge& \expect{\ftplvec\sim \mathbf{\alg}(\ftlvec), \ftplvec'\sim\mathbf{\alg}(\ftlvec')}{\frac{1}{2T}\sum_{t\in[2T]}\left(\left|\ftplp_t-A(\ftplvec)_t\right|+\left|\ftplp_t'-\hat{\ftplp}^0\right|\right)}\\
     \ge& \expect{\ftplvec\sim \mathbf{\alg}(\ftlvec), \ftplvec'\sim\mathbf{\alg}(\ftlvec')}{\frac{1}{2T}\sum_{t\in[T]}\left(\left|\ftplp_t-A(\ftplvec)_t\right|+\left|\ftplp_t'-\hat{\ftplp}^0\right|\right)}\\
     \ge& \expect{\ftplvec\sim \mathbf{\alg}(\ftlvec), \ftplvec'\sim\mathbf{\alg}(\ftlvec')}{\frac{1}{2T}\sum_{t\in[T]}\left|A(\ftplvec)_t-\hat{\ftlp}^0\right|}\\
     =& \frac{1}{4}\sqrt{\epsilon}+\epsilon.
\end{align*} Therefore, \begin{align*}
    &\max\left\{\expect{\ftplp\sim\mathbf{\alg}(\ftlvec) }{\ECE\left(\ftplp\right)}, \expect{\ftplp\sim\mathbf{\alg}(\ftlvec') }{\ECE\left(\ftplp\right)}\right\}\\
    \ge &\frac{1}{2}\expect{\ftplp\sim\mathbf{\alg}(\ftlvec) }{\ECE\left(\ftplp\right)}+\frac{1}{2}\expect{\ftplp\sim\mathbf{\alg}(\ftlvec') }{\ECE\left(\ftplp\right)}\ge \frac{1}{8}\sqrt{\epsilon}+\frac{1}{2}\epsilon.
\end{align*}
\end{proof}

\begin{theorem}
\label{thm: online lower bound cdl}
    For any post-processing algorithm $\alg$, there exists two sequences of predictions  $\ftlvec$ and $\ftlvec'$ with states $\vstate$ and $\vstate'$, respectively, both satisfying $\distcal(\ftlvec)=\distcal(\ftlvec')=\epsilon$, such that  
    \begin{equation*}
        \max\left\{\expect{}{\ECE\left(\ftplvec; \vstate\right)}, \expect{}{\ECE\left(\ftplvec';\vstate' \right)}\right\}\ge \frac{1}{8}\sqrt{\epsilon}+\frac{1}{2}\epsilon = \Theta(\sqrt{\epsilon}),
    \end{equation*}
    where we write $\ftplvec, \ftplvec'$ as the output of the post-processing algorithm $\alg$ on $\ftlvec, \ftlvec'$, respectively. 

    Moreover, the same argument holds for $\cdl$.
\end{theorem}

\begin{proof}[Proof of \Cref{thm: online lower bound cdl}]
    Define two sets of sequences of states corresponding to predictor $\ftlvec$ and $\ftlvec'$ that every $\vstate$ and $\vstate'$ in these sets are compatible with empirical posterior: $S_{\state}=\{\vstate\mid \sum_{t\in[T]}\state_t=T(\frac{1}{2}-\frac{1}{2}\sqrt{\epsilon}+\epsilon), \sum_{t=T+1}^{2T}\state_t=T(\frac{1}{2}+\frac{1}{2}\sqrt{\epsilon}-\epsilon)\}$, $S_{\state'}=\{\vstate\mid \sum_{t\in[2T]}\state_t=2T(\frac{1}{2}-\sqrt{\epsilon}-\epsilon)\}$. Denote the number of predicting prediction value $\ftplp\in \supp(\ftplvec)$ by $n_i=\sum_{t\in 2T}\ind{\ftplp_t=\ftplp_i}$.
    
    Fix a post-processing algorithm $\alg$.
    Define a proper scoring rule \begin{equation*}
        S_{\mu}(\ftplp,\state)=\left\{\begin{matrix} 
   \frac{1}{2}-\frac{1}{2}\cdot \frac{\state-\mu}{\max\{\mu,1-\mu\}}&& \mbox{if }\ftplp\le \mu\\  
   \frac{1}{2}+\frac{1}{2}\cdot \frac{\state-\mu}{\max\{\mu,1-\mu\}}&& \mbox{else.}
\end{matrix}\right. 
    \end{equation*}
    According to the definition of $\CDL$,
    \begin{align}\label{CDL to VBREG}
        \expect{\ftplvec\sim\mathbf{\alg}(\ftlvec) }{\CDL\left(\ftplvec, \vstate\right)}\ge \frac{1}{2T} \expect{\ftplvec\sim\mathbf{\alg}(\ftlvec) }{\sup_{\mu\in[0,1]}\sum_{t\in[2T]}\left(S_{\mu}(\hat{\ftplp_t},\state_t)-S_{\mu}(\ftplp_t,\state_t)\right)}.
    \end{align}
    For any sequence of predictions $\ftplvec$, define $N_{\ftplvec}=\sum_{t\in[T]}\ind{\ftplp_t\ge \mu}$, $M_{\ftplvec}=\sum_{t=T+1}^{2T}\ind{\ftplp_t\ge \mu}$.
    \begin{align*}
        &\expect{\vstate\in S_{\state} }{\sum_{t\in[2T]}\left(S_{\mu}(\hat{\ftplp_t},\state_t)-S_{\mu}(\ftplp_t,\state_t)\right)}\\
        \ge  &\expect{\vstate\in S_{\state} }{\frac{1}{\max\{\mu, 1-\mu\}}\sum_{\ftplp_i\in \supp(\ftplvec)}n_i\ind{\ftplp_i\le \mu}(\hat{\ftplp}_i-\mu)}\\
        =&\frac{1}{\max\{\mu, 1-\mu\}}\sum_{\ftplp_i\in \supp(\ftplvec), \ftplp_i\le \mu}\sum_{t\in[2T]}\ind{\ftplp_t=\ftplp_i}(\hat{\ftlp_t}-\mu)\\
        \ge& \frac{1}{\max\{\mu, 1-\mu\}}\sum_{\ftplp_i\in \supp(\ftplvec), \ftplp_i\le \mu}\sum_{t\in[2T]}\ind{\ftplp_t=\ftplp_i}\left(\frac{1}{2}-\frac{1}{2}\sqrt{\epsilon}+\epsilon-\mu\right)\\
        \ge & \frac{1}{\max\{\mu, 1-\mu\}}\left(2T-N_{\ftplvec}-M_{\ftplvec}\right)\left(\frac{1}{2}-\frac{1}{2}\sqrt{\epsilon}+\epsilon-\mu\right).
    \end{align*}
    \begin{align}
        &\expect{\vstate\in S_{\state}}{\expect{\ftplvec\sim\mathbf{\alg}(\ftlvec) }{\sup_{\mu\in[0,1]}\sum_{t\in[2T]}\left(S_{\mu}(\hat{\ftplp_t},\state_t)-S_{\mu}(\ftplp_t,\state_t)\right)}}\nonumber\\
        =&\expect{\ftplvec\sim\mathbf{\alg}(\ftlvec)}{\expect{\vstate\in S_{\state} }{\sup_{\mu\in[0,1]}\sum_{t\in[2T]}\left(S_{\mu}(\hat{\ftplp_t},\state_t)-S_{\mu}(\ftplp_t,\state_t)\right)}}\nonumber\\
        \ge & \expect{\ftplvec\sim\mathbf{\alg}(\ftlvec)}{\sup_{\mu\in[0,1]}\expect{\vstate\in S_{\state} }{\sum_{t\in[2T]}\left(S_{\mu}(\hat{\ftplp_t},\state_t)-S_{\mu}(\ftplp_t,\state_t)\right)}}\nonumber\\
        \label{CDL lower bound q1}\ge & \expect{\ftplvec\sim\mathbf{\alg}(\ftlvec)}{\sup_{\mu\in[0,1]}\frac{1}{\max\{\mu, 1-\mu\}}\left(2T-N_{\ftplvec}-M_{\ftplvec}\right)\left(\frac{1}{2}-\frac{1}{2}\sqrt{\epsilon}+\epsilon-\mu\right)}.
    \end{align}
    Combine inequality (\ref{CDL to VBREG}) and (\ref{CDL lower bound q1}), there exists $\vstate\in S_{\state}$, that \begin{equation}\label{CDL q1}\begin{aligned}
        &\expect{\ftplvec\sim\mathbf{\alg}(\ftlvec) }{\CDL\left(\ftplvec, \vstate\right)}\\
        \ge &\frac{1}{2T}\expect{\ftplvec\sim\mathbf{\alg}(\ftlvec)}{\sup_{\mu\in[0,1]}\frac{1}{\max\{\mu, 1-\mu\}}\left(2T-N_{\ftplvec}-M_{\ftplvec}\right)\left(\frac{1}{2}-\frac{1}{2}\sqrt{\epsilon}+\epsilon-\mu\right)}.
    \end{aligned}
    \end{equation}
    \begin{align*}
        &\expect{\vstate'\in S_{\state'} }{\sum_{t\in[2T]}\left(S_{\mu}(\hat{\ftplp_t},\state_t)-S_{\mu}(\ftplp_t,\state_t)\right)}\\
        \ge  &\expect{\vstate\in S_{\state} }{\frac{1}{\max\{\mu, 1-\mu\}}\sum_{\ftplp_i\in \supp(\ftplvec')}n_i\ind{\ftplp_i\ge \mu}(\mu-\hat{\ftplp}_i)}\\
        =&\frac{1}{\max\{\mu, 1-\mu\}}\sum_{\ftplp_i\in \supp(\ftplvec'), \ftplp_i\ge \mu}\sum_{t\in[2T]}\ind{\ftplp_t=\ftplp_i}(\mu-\hat{\ftlp_t'})\\
        =& \frac{1}{\max\{\mu, 1-\mu\}}\sum_{\ftplp_i\in \supp(\ftplvec'), \ftplp_i\ge \mu}\sum_{t\in[2T]}\ind{\ftplp_t=\ftplp_i}\left(\mu-\frac{1}{2}+\sqrt{\epsilon}+\epsilon\right)\\
        =& \frac{1}{\max\{\mu, 1-\mu\}}\left(N_{\ftplvec'}+M_{\ftplvec'}\right)\left(\mu-\frac{1}{2}+\sqrt{\epsilon}+\epsilon\right).
    \end{align*}
    Similarly, there exists $\vstate'\in S_{\state'}$, that \begin{equation}
        \label{CDL q2}\begin{aligned}&\expect{\ftplvec'\sim\mathbf{\alg}(\ftlvec') }{\CDL\left(\ftplvec', \vstate'\right)}\\
        \ge &\frac{1}{2T}\expect{\ftplvec'\sim\mathbf{\alg}(\ftlvec')}{\sup_{\mu\in[0,1]}\frac{1}{\max\{\mu, 1-\mu\}}\left(N_{\ftplvec'}+M_{\ftplvec'}\right)\left(\mu-\frac{1}{2}+\sqrt{\epsilon}+\epsilon\right)}.
        \end{aligned}
    \end{equation}
    For any $\ftplvec\sim \mathbf{\alg}(\ftlvec)$ and $\ftplvec'\sim \mathbf{\alg}(\ftlvec')$, $\ftplp_t=\ftplp_{t}'$ always holds for $t\in[T]$, since $\ftlp_t=\ftlp_t'$ always holds for $t\in[T]$. So $N_{\ftplvec}=N_{\ftplvec'}$ and $M_{\ftplvec}=M_{\ftplvec'}$ always hold for $t\in[T]$.
    Combine inequality (\ref{CDL q1}) and (\ref{CDL q2}), \begin{align}
        &\max\left\{\expect{\ftplvec\sim\mathbf{\alg}(\ftlvec) }{\CDL\left(\ftplvec, \vstate\right)}, \expect{\ftplvec'\sim\mathbf{\alg}(\ftlvec') }{\CDL\left(\ftplvec', \vstate'\right)}\right\}\nonumber \\
        \ge& \frac{1}{4T}\expect{\ftplvec\sim\mathbf{\alg}(\ftlvec)}{\sup_{\mu\in[0,1]}\frac{1}{\max\{\mu, 1-\mu\}}\left(2T-N_{\ftplvec}-M_{\ftplvec}\right)\left(\frac{1}{2}-\frac{1}{2}\sqrt{\epsilon}+\epsilon-\mu\right)}\nonumber\\
        +&\frac{1}{4T}\expect{\ftplvec'\sim\mathbf{\alg}(\ftlvec')}{\sup_{\mu\in[0,1]}\frac{1}{\max\{\mu, 1-\mu\}}\left(N_{\ftplvec'}+M_{\ftplvec'}\right)\left(\mu-\frac{1}{2}+\sqrt{\epsilon}+\epsilon\right)}\label{max CDL 1}\\
        \ge& \frac{1}{4T}\expect{\ftplvec\sim \mathbf{\alg}(\ftlvec), \ftplvec'\sim\mathbf{\alg}(\ftlvec')}{\frac{1}{\frac{1}{2}+\frac{3}{4}\sqrt{\epsilon}}\left(2T-M_{\ftplvec}+M_{\ftplvec'}\right)\left(\frac{1}{4}\sqrt{\epsilon}+\epsilon\right)}\label{max CDL 2}\\
        \ge & \frac{1}{8}\sqrt{\epsilon}+\frac{1}{2}\epsilon. \label{max CDL 3}
    \end{align}
    By taking $\mu=\frac{1}{2}-\frac{3}{4}\sqrt{\epsilon}$ for both cases for $\ftplvec$ and $\ftplvec'$ and get (\ref{max CDL 1})$\ge$(\ref{max CDL 2}). Since $M_{\ftplvec}, M_{\ftplvec'}\in[0,T]$, $M_{\ftplvec'}-M_{\ftplvec}\ge -T$, so (\ref{max CDL 2})$\ge$(\ref{max CDL 3}).
\end{proof}

\end{document}